\documentclass[sigconf]{acmart}
\AtBeginDocument{%
  }

\setcopyright{acmlicensed}
\copyrightyear{2018}
\acmYear{2018}
\acmDOI{XXXXXXX.XXXXXXX}
\acmConference[Conference acronym 'XX]{Make sure to enter the correct
  conference title from your rights confirmation email}{June 03--05,
  2018}{Woodstock, NY}
\acmISBN{978-1-4503-XXXX-X/2018/06}

\usepackage{makecell}
\usepackage[table]{xcolor}
\usepackage{multirow}
\usepackage{threeparttable}
\usepackage{subfig}
\usepackage{amsmath,amsthm} 
\newcommand{\B}[1]{\mathbf{#1}}      
\newcommand{\M}[1]{\boldsymbol{#1}}      

\newcommand{\lrA}[1]{\ensuremath{\left(#1\right)}}

\newtheorem{theorem}{Theorem}[section]
\newtheorem{proposition}[theorem]{Proposition}
\newtheorem{assumption}[theorem]{Assumption}
\makeatletter
\apptocmd{\@mkabstract}{\vspace{-0.5\baselineskip}}{}{}
\makeatother

\begin{document}

\title{Hypernetwork-Parameterized Spatially Adaptive Neural Operators for PDE Learning}

\settopmatter{authorsperrow=5}

\author{Jiaquan Zhang}
\affiliation{%
  \institution{School of Information and Software Engineering, UESTC}
  \city{Chengdu}
  \country{China}
}

\author{Chaoning Zhang}
\affiliation{%
  \institution{School of Computer Science and Engineering, UESTC}
  \city{Chengdu}
  \country{China}
}

\author{Shuxu Chen}
\affiliation{%
  \institution{Electronics and Information Convergence Engineering, KHU}
  \city{Yongin-si}
  \country{Korea}}

\author{Meng Ye}
\affiliation{%
  \institution{Department of Mechanical Engineering, THU}
  \city{Beijing}
  \country{China}
}

\author{Xiaofeng~Zhang}
\affiliation{%
 \institution{School of Automation and Intelligent Sensing, SJTU}
 \city{Shanghai}
 \country{China}}


\author{Qiang He}
\affiliation{%
  \institution{Department of Mechanical Engineering, THU}
  \city{Beijing}
  \country{China}}

\author{Weifeng Huang}
\affiliation{%
  \institution{Department of Mechanical Engineering, THU}
  \city{Beijing}
  \country{China}}

\author{Guoqing Wang}
\affiliation{%
  \institution{School of Computer Science and Engineering, UESTC}
  \city{Chengdu}
  \country{China}}

\author{Yang Yang}
\affiliation{%
  \institution{School of Computer Science and Engineering, UESTC}
  \city{Chengdu}
  \country{China}
}

\author{Caiyan Qin}
\affiliation{%
  \institution{School of Robotics and Advanced Manufacture, HIT}
  \city{Shenzhen}
  \country{China}
}

\renewcommand{\shortauthors}{Zhang et al.}

\begin{abstract}
Spatially heterogeneous partial differential equations (PDEs) exhibit location-dependent dynamics arising from variations in geometry and physical coefficients. Existing neural operators improve localized modeling through multiscale features, attention mechanisms, or domain decomposition, yet their update rules often remain spatially shared. 
Hypernetwork-based methods adapt parameters across PDE instances but typically generate only one global parameterization per instance. 
Consequently, shared operators may underfit boundaries and high-gradient regions, with these localized errors accumulating during autoregressive rollout.
We propose a spatially adaptive neural operator (SANO), which replaces this spatially shared parameterization with a spatially continuous field of location-dependent operator parameters. 
SANO uses Fourier-encoded coordinates and a coordinate-conditioned hypernetwork to generate spatial operator-conditioning codes at sampling points. A Hyper-Neural Element (HNE) mechanism interpolates these codes within local subregions, coupling neighboring operators while allowing their update rules to vary across space, and partition-of-unity weights assemble the overlapping local predictions.
Experiments on one-, two-, and three-dimensional PDEs and two perforated-domain elliptic benchmarks show that SANO consistently outperforms competitive neural-operator, hypernetwork-based, and physics-informed baselines.
\end{abstract}


\received{20 February 2007}
\received[revised]{12 March 2009}
\received[accepted]{5 June 2009}

\maketitle

\section{Introduction}
Partial differential equations (PDEs) describe physical processes in computational fluid dynamics, materials simulation, geophysics, and engineering design \cite{brunton2024promising,bezgin2023jax,zhang2022analyses,degen2023perspectives,wu2022learning}. 
Classical numerical methods solve PDEs through numerical discretization, but a new solve is generally required when the initial conditions, boundary conditions or source terms change~\cite{le2025learning}. 
Neural-network-based solvers have therefore become an important complement to traditional simulation pipelines~\cite{zhu2026physicssolver}. 
Neural operators instead learn mappings from PDE inputs to solution fields, allowing a trained model to predict solutions for unseen problem instances without repeating the full numerical solution process \cite{kovachki2023neural,yang2023context,zhou2026tf,zhang2026autoregression}.

The Fourier neural operator (FNO) \cite{DBLP:conf/iclr/LiKALBSA21} performs spectral convolution in the Fourier domain, enabling efficient operator learning while capturing large-scale solution structures. 
Subsequent methods improve neural operators along complementary directions: PINO \cite{li2024physics} incorporates PDE residuals during training, attention-based operators such as GNOT \cite{hao2023gnot} model long-range spatiotemporal interactions, global--local architectures introduce local components to recover high-frequency details \cite{kalimuthu2025loglo}, and adaptive Fourier models have been explored for large-scale forecasting \cite{pathak2022fourcastnet}. 
Despite their different designs, these methods largely retain a spatially shared operator parameterization, applying one set of parameters across the entire domain.
This parameter sharing becomes restrictive for PDEs with spatially varying coefficients or heterogeneous media, where different regions exhibit distinct local dynamics \cite{englert2025spatially,lippe2023pde}. 
A single operator kernel applies the same update rule across regions with distinct local behavior, often missing localized sharp structures. 
The resulting local errors can further accumulate during autoregressive rollout \cite{lippe2023pde}. 
The analysis in Section~\ref{sec:motivation} provides direct evidence of localized underfitting. As shown in Figures~\ref{fig:motivation_field} and~\ref{fig:global_proxy_error}, global proxies fit smooth regions well but concentrate their errors near geometric boundaries and transition regions, where the solution exhibits sharper spatial variations. This correspondence reveals a mismatch between globally shared operator parameters and spatially varying solution complexity, motivating operator parameters that vary continuously with location.
Several recent methods introduce adaptive mechanisms to better capture region-dependent behavior in heterogeneous PDEs. 
Attention-based modulation \cite{hao2023gnot} and domain decomposition adapt features or computation across regions, while hypernetwork-based operators such as HyperDeepONet \cite{lee2023hyperdeeponet} and HyPINO \cite{bischof2026hypino} generate target-network parameters from the input function or PDE specification. 
These methods adapt the model to each problem instance, but they generate one parameter set for the entire domain, leaving operator parameters that vary continuously with spatial location largely unaddressed.

To address the above problem, we propose a spatially adaptive neural operator (SANO). Rather than applying one shared operator across the whole domain, SANO instantiates a location-dependent local operator at each query point, so that regions with distinct dynamics are modeled by distinct local behaviors. A single local operator is shared across all locations, and its behavior at each point is determined by a compact location-dependent code produced from the coordinate, which avoids maintaining an independent operator per location.
Specifically, SANO encodes the coordinate with Fourier features and passes it to a hypernetwork that generates codes at a set of sampling points. A Hyper-Neural Element (HNE) mechanism interpolates these anchor codes into a spatially continuous code field, coupling neighboring locations while preserving local flexibility. Conditioned on the code at each query point, the shared operator produces a local prediction, and partition-of-unity weights blend the overlapping predictions into a globally consistent output. A smoothness regularizer on the code field further stabilizes training and long-horizon rollout, while at inference the coordinates alone generate the codes that condition the shared operator, instantiating the local operators without any additional supervision.
The main contributions of this work are summarized as follows:

\textbf{\textit{(i)}} We propose SANO, a spatially adaptive neural operator that represents the solution operator as a family of location-dependent operators whose local behavior is set by a spatially varying conditioning code, enabling local adaptation while preserving global consistency.
\textbf{\textit{(ii)}} We introduce a hyper-neural element mechanism that constructs spatial conditioning codes for the local operator from shared sampling-point codes and interpolates them in code space, enforcing structured coupling and smooth transitions across neighboring regions.
\textbf{\textit{(iii)}} We introduce a code-field smoothness regularizer that suppresses spatial oscillation of the generated conditioning codes, improving predictive accuracy and rollout stability.
\textbf{\textit{(iv)}} We validate SANO on one-, two-, and
three-dimensional PDE benchmarks, where it improves single-step accuracy, long-horizon rollout stability, and cross-parameter and cross-resolution generalization over competitive baselines.

\begin{figure}[t]
    \centering
    \includegraphics[width=0.99\linewidth]{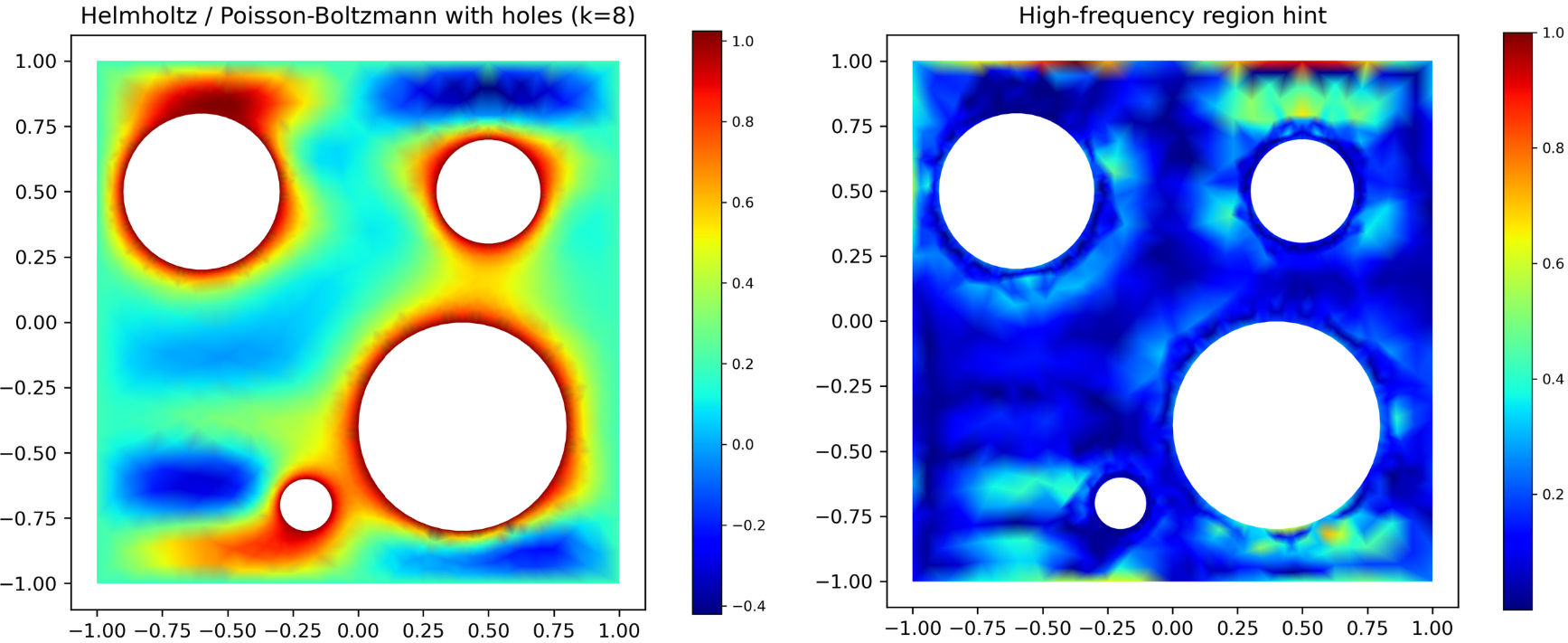}
    \caption{Nonuniform field complexity. Left: target field of a HZ-G problem with holes. Right: high-frequency region hint showing localized regions with stronger spatial variations.}
    \label{fig:motivation_field}
\end{figure}

\begin{figure}[t]
\centering
\includegraphics[width=0.99\linewidth]{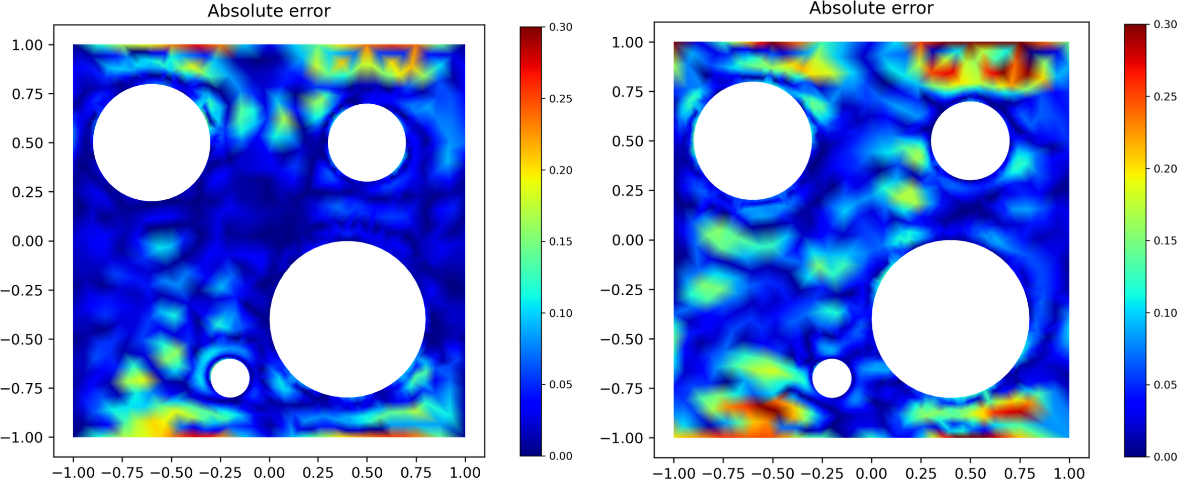}
\caption{Spatial concentration of approximation errors.
The Fourier proxy (left) and MLP proxy (right) produce larger errors near geometric boundaries and high-frequency regions.}
\label{fig:global_proxy_error}
\end{figure}

\section{Related Work}
Neural operators overcome the instance-wise optimization cost of physics-informed neural networks~\cite{raissi2019physics,gao2021phygeonet} by learning function-to-function mappings that generalize across varying initial conditions, boundary conditions, and PDE parameters. DeepONet~\cite{lu2021learning} uses a branch--trunk architecture, whereas the FNO~\cite{DBLP:conf/iclr/LiKALBSA21} applies global frequency-domain convolutions for resolution-independent prediction; physics-informed variants such as PINO~\cite{li2024physics} and PI-DeepONet~\cite{wang2021learning} further incorporate PDE residuals during training. However, most neural operators rely on globally shared parameters, implicitly assuming spatial homogeneity and limiting their accuracy on spatially heterogeneous PDEs. Attention-based~\cite{hemmasian2024multi} and domain-decomposition~\cite{zhou2025rapid} methods adapt features or computation across regions but largely retain shared operator parameters, while hypernetwork approaches such as HyperDeepONet~\cite{lee2023hyperdeeponet} and HyPINO~\cite{bischof2026hypino} generate one parameter set per problem instance rather than spatially varying parameters.
The complete related work section is shown in Appendix~\ref{app:rl}

\section{Motivation}

\label{sec:motivation}

\begin{figure*}[t]
    \centering
    \includegraphics[width=0.85\linewidth]{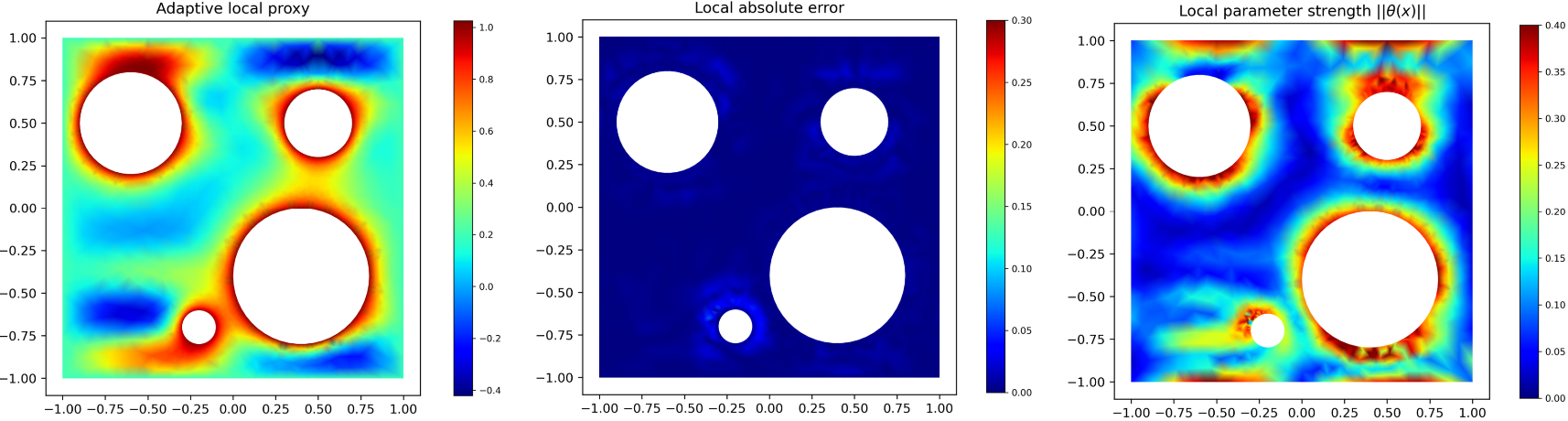}
    \caption{Adaptive local responses. 
    Left: learned adaptive local proxy. 
    Middle: local absolute error of the adaptive model. 
    Right: learned local adaptation strength, with higher values around geometrically complex regions.}
    \label{fig:adaptive_local_response}
\end{figure*}

We consider a representative Helmholtz complex-geometry (HZ-G) benchmark on a perforated domain to illustrate why an operator should adapt its behavior to local solution complexity. 
Although the solution remains smooth over most of the domain, sharper variations arise near geometric boundaries and transition regions. 
We examine this spatial nonuniformity through i) the complexity of the target field, ii) the localized errors of globally shared proxies, and iii) the response of a spatially adaptive model.

\subsection{Nonuniform field complexity}
In complex geometries or multiscale physical systems, the modeling difficulty of a solution field is often spatially nonuniform. 
As shown in Figure~\ref{fig:motivation_field}, the left panel presents a representative HZ-G field with holes, where most interior regions are relatively smooth. At the same time, pronounced local variations appear near hole boundaries, outer boundaries, and transition regions. 
The right panel further highlights this nonuniformity through a high-frequency region hint, showing that stronger spatial variations are mainly concentrated around geometrically complex boundaries. 
This suggests that a uniform operator representation over the entire domain may be inefficient or insufficient: smooth regions require only simple local behavior, whereas rapidly varying regions require stronger local expressiveness. 
Therefore, an effective neural operator should adapt its operator behavior according to spatial location rather than relying on a single globally shared parameterization.

\subsection{Localized errors of global proxies}
To further examine the limitation of globally shared representations, we fit the HZ-G field with holes using two global proxy models and visualize their absolute error maps in Figure~\ref{fig:global_proxy_error}. 
The left panel shows the error of a global Fourier proxy, while the right panel shows the error of a global MLP proxy. 
Although both proxies can approximate most smooth regions, their errors are not uniformly distributed over the domain. 
Instead, large errors are concentrated near hole boundaries, outer boundaries, and localized high-frequency regions. 
This indicates that the main modeling difficulty arises from local geometric complexity rather than uniform global mismatch. 
Therefore, a single globally shared representation may underfit complex local regions, motivating spatially adaptive operator behavior that provides stronger local expressiveness where needed.

\subsection{Adaptive local responses}
We further examine whether the spatially adaptive model learns a localized fitting strategy on the HZ-G problem with holes. 
As shown in Figure~\ref{fig:adaptive_local_response}, the left panel presents the learned adaptive local proxy, where stronger responses appear near hole boundaries, outer boundaries, and transition regions. 
These regions are consistent with the high-frequency and large-gradient structures observed in the target field. 
The middle panel shows the local absolute error, which remains low over most of the domain and is less concentrated than the errors of global proxy models. 
The right panel visualizes the learned local adaptation strength, where higher values are assigned around geometrically complex regions and lower values appear in smoother areas.
This indicates that the model does not use the same level of complexity everywhere, but allocates stronger local expressiveness to regions with higher physical and geometric complexity.

Taken together, these observations indicate that spatially nonuniform solution complexity cannot be fully addressed by feature-level adaptation alone; the local input--output mapping itself should vary across the domain through spatially adaptive operator behavior.

\begin{figure*}[t]
    \centering
    \includegraphics[width=0.82\linewidth]{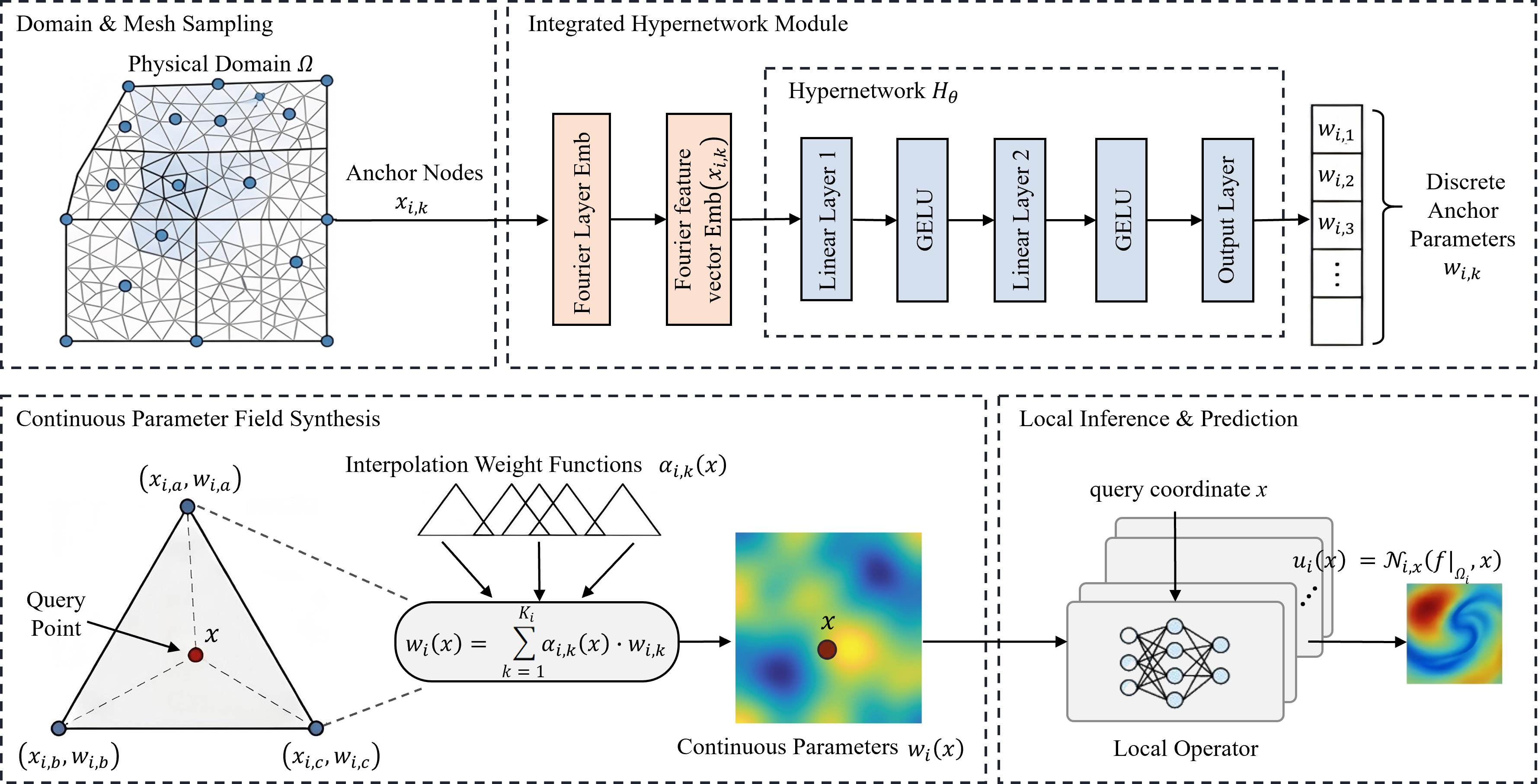}
    \caption{Overall architecture of SANO, including coordinate-conditioned code generation, code-space interpolation, code-conditioned local prediction, and prediction of the next solution field.}
    \label{fig:sano_overview}
\end{figure*}

\section{Method}
\label{sec:sano}

SANO replaces the globally shared parameterization of a standard neural operator with a continuously varying, location-dependent code field that conditions a shared local operator (see Figure\ref{fig:sano_overview}).
Instead of applying a single update rule across the entire domain, SANO conditions this shared operator on a location-dependent code, so that smooth regions and geometrically complex regions are modeled by different local behaviors.
In this section, $\theta$ denotes the learnable parameters of the
hypernetwork and $\nu$ the parameters of the shared local MLP, both of
which stay fixed after training. In contrast, $w(x)$ denotes the spatial
conditioning code generated at location $x$, which varies across space and
modulates the shared local MLP. The superscript $n$ is a time-step index.

\subsection{Problem Setup}
\label{sec:problem_setup}

Let $\Omega\subset\mathbb{R}^{d}$ be a compact spatial domain, and let $X$
and $Y$ denote the input and output function spaces defined over $\Omega$.
A solution operator $\mathcal{G}:X\rightarrow Y$ maps an input function
$f\in X$ to the corresponding solution $u=\mathcal{G}(f)\in Y$.
Conventional neural operators adopt the globally shared form
$\widehat{u}(x)=\mathcal{G}_{\Theta}(f)(x)$, where $\Theta$ is one parameter set shared over all of $\Omega$. This form assumes that every spatial location follows the same update rule. 
SANO represents the solution operator as a
family of location-dependent local operators,
\begin{equation}
\mathcal{G}(f)(x)\approx\mathcal{G}_{x}(f),
\qquad x\in\Omega,
\label{eq:adaptive_family}
\end{equation}
where $\mathcal{G}_{x}$ is the local operator associated with location $x$,
and the mapping $x\mapsto\mathcal{G}_{x}$ is constrained to vary smoothly in code space.
The parameters $\nu$ of the local operator are shared across locations, while $w(x)$ selects which local behavior applies rather than emitting a solution value. 
SANO therefore learns the coordinate-to-code map $x\mapsto w(x)$.

\subsection{Coordinate Encoding}
\label{sec:feature_representation}

To expose multiscale spatial information to the hypernetwork, SANO first encodes the coordinate with Fourier features. 
For any $x\in\mathbb{R}^{d}$, the spatial embedding
$\mathrm{Emb}:\Omega\rightarrow\mathcal{Z}$ is defined as
\begin{equation}
\mathrm{Emb}(x)
=
\bigl[\sin(2\pi\B{B}x),\ \cos(2\pi\B{B}x)\bigr]
\in\mathbb{R}^{2m},
\label{eq:fourier_emb}
\end{equation}
where $\B{B}\in\mathbb{R}^{m\times d}$ is a random frequency matrix whose
entries are drawn independently from the Gaussian distribution
$\mathrm{N}(0,\sigma_{B}^{2})$; $m$ controls the embedding capacity;
$\sigma_{B}$ controls the highest spatial frequency the hypernetwork can
resolve; and $\mathcal{Z}=\mathbb{R}^{2m}$ denotes the positional encoding
space.
The encoding supplies spatial conditioning for code generation rather
than regressing $u(x)$, so it shapes the behavior of the local operator
instead of the predicted value at a single point.

\subsection{Hypernetwork Code Generation}
\label{sec:parameter_generation}

To turn a positional encoding into a location-dependent local operator, SANO composes a coordinate-conditioned hypernetwork with a shared local MLP. Let $H_{\theta}$ denote the hypernetwork. To generate the spatial parameter code in a single forward pass, SANO evaluates $H_{\theta}$ at the embedded coordinate,
\begin{equation}
w(x):=H_{\theta}\!\left(\mathrm{Emb}(x)\right)\in\mathbb{R}^{p},
\label{eq:param_field_def}
\end{equation}
where $w(x)$ is a compact spatial parameter code that conditions the shared local MLP at location $x$, rather than the full weight set of a location-specific operator or a solution value.
The generated code is concatenated with the local input representation and used to condition a shared local MLP $G_{\nu}$, so that the prediction is
\begin{equation}
\widehat{u}(x)=G_{\nu}\!\left(f,x,w(x)\right),
\label{eq:local_model_def}
\end{equation}
where the backbone parameters $\nu$ of $G_{\nu}$ are shared across all spatial locations, while different codes $w(x)$ induce different effective local transformations.

\subsection{Hyper-Neural Element}
\label{sec:hne}

Generating $w(x)$ independently at every query location yields strong local adaptivity but lets the code of neighboring locations change abruptly, which induces discontinuities and numerical oscillation in the predicted
field. 
To balance spatial adaptivity against global smoothness, the
hyper-neural element (HNE) construction generates code anchors at a finite set of sampling points and interpolates them into a continuous parameter-code field.

\paragraph{Domain Decomposition and Anchor Generation}
\label{sec:hne:domain}

To localize the parameterization, SANO covers the domain by a finite collection of subregions $\{\Omega_{i}\}_{i=1}^{M}$ with $\Omega\subset\bigcup_{i=1}^{M}\Omega_{i}$. The subregions are allowed to overlap and are not required to form a strict partition, and the
decomposition discretizes neither the solution nor the governing equation.
Within each subregion, SANO selects a finite sampling set
$S_{i}=\{x_{i,k}\}_{k=1}^{K_{i}}\subset\Omega_{i}$ and evaluates the
hypernetwork at each sampling location,
\begin{equation}
w_{i,k}:=H_{\theta}\bigl(\mathrm{Emb}(x_{i,k})\bigr)\in\mathbb{R}^{p},
\qquad k=1,\dots,K_{i},
\label{eq:anchor_param}
\end{equation}
where $K_{i}$ is the number of sampling points in $\Omega_{i}$ and
$w_{i,k}$ is a code anchor describing how the local operator should
behave near $x_{i,k}$.
The resulting sampling-point code set is
$\mathcal{D}_{i}=\{(x_{i,k},w_{i,k})\}_{k=1}^{K_{i}}$.

\subsection{Parameter Interpolation}
\label{sec:hne:interpolation}

To obtain the code at an arbitrary query location without generating it
independently, SANO introduces a code interpolation operator
$I_{i}:\mathcal{D}_{i}\rightarrow(\Omega_{i}\rightarrow\mathbb{R}^{p})$ and
sets $w_{i}(x):=I_{i}(\mathcal{D}_{i})(x)$. Instantiated as a linear
interpolant,
\begin{equation}
w_{i}(x)
=
\sum_{k=1}^{K_{i}}\alpha_{i,k}(x)\,w_{i,k},
\qquad x\in\Omega_{i},
\label{eq:param_interpolation}
\end{equation}
where $\alpha_{i,k}:\Omega_{i}\rightarrow\mathbb{R}$ are interpolation
weight functions satisfying $\sum_{k=1}^{K_{i}}\alpha_{i,k}(x)=1$ for every
$x\in\Omega_{i}$. 
Requiring in addition that $\alpha_{i,k}(x)\geq 0$ makes
$w_{i}(x)$ a convex combination of the anchors, which prevents the query code from departing far from its neighborhood and improves stability.
The interpolation acts in code space rather than in solution space.
SANO interpolates the anchors into $w_{i}(x)$ and then
conditions the shared local operator on it. 
The local mappings at different locations of
the same subregion are therefore coupled through shared anchors while retaining flexibility for local adaptation.
We provide a one-dimensional linear element example in Appendix~\ref{One-dimensional} to better comprehend the HNE.

\subsection{Conditional Local Evaluation}
\label{sec:hne:realization}
Given the position-dependent code $w_i(x)$ obtained by
code-space interpolation, SANO uses it to condition a shared local
MLP. The local input feature within subregion $\Omega_{i}$ is
$q_{i}(x)=\operatorname{Enc}(f|_{\Omega_{i}},x)$, where $f|_{\Omega_{i}}$ is
the restriction of the input function to $\Omega_{i}$ and
$\operatorname{Enc}(\cdot)$ encodes the local input together with the query
coordinate. This feature is concatenated with the parameter encoding to form
the local-MLP input,
\begin{equation}
h_{i}^{(0)}(x)=\bigl[\,q_{i}(x);\,w_{i}(x)\,\bigr].
\label{eq:local_mlp_input}
\end{equation}
With shared parameters $\nu=\{W^{(\ell)},b^{(\ell)}\}_{\ell=1}^{L}$, the
hidden layers update as
$h_{i}^{(\ell)}(x)=\sigma(W^{(\ell)}h_{i}^{(\ell-1)}(x)+b^{(\ell)})$ for
$\ell=1,\ldots,L-1$, and the output layer gives the local prediction
$u_{i}(x)=W^{(L)}h_{i}^{(L-1)}(x)+b^{(L)}$ for $x\in\Omega_{i}$. For
compactness, we write this position-conditioned local mapping as
\begin{equation}
u_{i}(x)=G_{\nu}\bigl(f|_{\Omega_{i}},x,w_{i}(x)\bigr),
\label{eq:local_realization}
\end{equation}
where all locations share the local-MLP parameters $\nu$, while the
position-dependent encoding $w_{i}(x)$ determines the local input--output
behavior at each location. Because $w_{i}(x)$ is continuously interpolated
from neighboring code anchors, the local mapping varies continuously
across adjacent query locations.
For time-dependent PDEs, SANO learns a one-step time-marching operator: given
the current state $u^{n}$, the next-step local prediction in $\Omega_{i}$ at
$x$ is
$\widehat{u}_{i}^{\,n+1}(x)=G_{\nu}(u^{n}|_{\Omega_{i}},x,w_{i}(x))$.
Fusing the overlapping local predictions with the partition-of-unity weights
yields the global one-step map
$\widehat{u}^{\,n+1}=\mathcal{G}_{\mathrm{SANO}}(\widehat{u}^{\,n})$, which is
applied autoregressively to produce the long-horizon rollout. Since the
parameter encoding depends only on spatial position, the same location reuses
the same spatial conditioning across time steps, while the local MLP maps the
current local state to its time evolution. The model can thus form distinct
local input--output mappings across regions: the encoding varies slowly in
smooth regions, whereas boundaries, interfaces, and high-frequency structures
receive different encodings, strengthening the expressiveness of SANO for
spatially heterogeneous dynamics.

\subsection{Global Assembly via Partition of Unity}
\label{sec:hne:pou}

Because the subregions overlap, a location $x\in\Omega_{i}\cap\Omega_{j}$
receives predictions $u_{i}(x)$ and $u_{j}(x)$ from more than one local
model, and selecting one of them introduces a discontinuity at subregion
boundaries. SANO therefore blends the local predictions with a partition of
unity (PoU). Let $\{\phi_{i}\}_{i=1}^{M}\subset C(\Omega,[0,1])$ be
nonnegative weighting functions with
$\operatorname{supp}(\phi_{i})\subset\Omega_{i}$ satisfying
\begin{equation}
\sum_{i=1}^{M}\phi_{i}(x)=1,
\qquad\forall x\in\Omega,
\label{eq:pou_def}
\end{equation}
where $\phi_{i}(x)$ is the contribution weight of subregion $\Omega_{i}$ at
$x$, and $\phi_{i}(x)=0$ whenever $x$ lies outside the support of
$\phi_{i}$. The global prediction is the weighted combination
\begin{equation}
\widehat{u}(x):=\sum_{i=1}^{M}\phi_{i}(x)\,u_{i}(x),
\qquad x\in\Omega .
\label{eq:global_output}
\end{equation}
Substituting Eqs.~\eqref{eq:anchor_param}--\eqref{eq:local_realization} into
Eq.~\eqref{eq:global_output} gives the complete model,
\begin{equation}
\widehat{u}(x)
=
\sum_{i=1}^{M}\phi_{i}(x)\,
G_{\nu}\!\left(
f|_{\Omega_{i}},x,
\sum_{k=1}^{K_{i}}\alpha_{i,k}(x)\,
H_{\theta}\bigl(\mathrm{Emb}(x_{i,k})\bigr)
\right),
\label{eq:sano_full}
\end{equation}
which exposes the four stages of SANO in order: the hypernetwork generates
anchor codes at the sampling points, HNE interpolates them into a continuous
spatial parameter-code field, the shared local MLP produces conditional local
predictions, and the PoU weights assemble them into a global solution.
It is worth noting that, the PoU assembly does not amplify the largest local error. Let $u^{\star}$
denote the target solution and suppose each local prediction satisfies
$|u_{i}(x)-u^{\star}(x)|\leq\delta_{i}$. Using $\phi_{i}\geq 0$ and
Eq.~\eqref{eq:pou_def},
\begin{equation}
\begin{aligned}
\bigl|\widehat{u}(x)-u^{\star}(x)\bigr|
&=\Bigl|\sum_{i=1}^{M}\phi_{i}(x)\bigl(u_{i}(x)-u^{\star}(x)\bigr)\Bigr|\\
&\leq\sum_{i=1}^{M}\phi_{i}(x)\bigl|u_{i}(x)-u^{\star}(x)\bigr|\\
&\leq\sum_{i=1}^{M}\phi_{i}(x)\,\delta_{i}
\;\leq\;\max_{1\leq i\leq M}\delta_{i},
\end{aligned}
\label{eq:pou_bound}
\end{equation}
where $\delta_{i}$ is the uniform local error bound on $\Omega_{i}$.
The global assembly is therefore no worse than the worst local prediction,
which is why the PoU fuses overlapping local models stably.
The uniform-approximation analysis built on this bound is given in
Appendix~\ref{sec:approx_theory}.

\subsection{Training Objective}
\label{sec:training_objective}
SANO is trained by combining a data-fidelity term with a regularizer that
controls the spatial smoothness of the generated code field $w(x)$,
which suppresses oscillation of the code and improves long-horizon rollout
stability. The total objective combines a data term and a code-field
smoothness regularizer,
\begin{equation}
\mathcal{L}
=
\mathcal{L}_{\mathrm{data}}
+\lambda_{\mathrm{reg}}\mathcal{L}_{\mathrm{reg}},
\label{eq:loss_total}
\end{equation}
where $\lambda_{\mathrm{reg}}$ weights the regularization term. The two
terms are
\begin{subequations}
\label{eq:loss_terms}
\begin{align}
\mathcal{L}_{\mathrm{data}}
&=\frac{1}{N_{\mathrm{d}}}\sum_{j=1}^{N_{\mathrm{d}}}
\bigl\|\widehat{u}(x_{j})-u^{\star}(x_{j})\bigr\|_{2}^{2},
\label{eq:loss_data}\\
\mathcal{L}_{\mathrm{reg}}
&=\frac{1}{N_{\mathrm{r}}}\sum_{j=1}^{N_{\mathrm{r}}}
\bigl\|\nabla_{x}w(x_{j})\bigr\|_{2}^{2},
\label{eq:loss_reg}
\end{align}
\end{subequations}
where $u^{\star}$ is the reference solution;
$\{x_{j}\}_{j=1}^{N_{\mathrm{d}}}$ are the supervised data points;
$\{x_{j}\}_{j=1}^{N_{\mathrm{r}}}$ are the collocation points at which the
smoothness regularizer is evaluated; and $\nabla_{x}$ denotes the gradient
of the generated code field, which is distinct from the discrete
gradient $\nabla_{h}$ used in the evaluation metrics.
Under the HNE construction, Eq.~\eqref{eq:param_interpolation} gives
$\nabla_{x}w_{i}(x)=\sum_{k=1}^{K_{i}}\nabla_{x}\alpha_{i,k}(x)\,w_{i,k}$,
so the gradient of the code field is carried entirely by the
interpolation weights, and Eq.~\eqref{eq:loss_reg} admits the local weighted
form
\begin{equation}
\mathcal{L}_{\mathrm{reg}}
=\frac{1}{N_{\mathrm{r}}}\sum_{j=1}^{N_{\mathrm{r}}}\sum_{i=1}^{M}
\phi_{i}(x_{j})\bigl\|\nabla_{x}w_{i}(x_{j})\bigr\|_{2}^{2},
\label{eq:loss_reg_hne}
\end{equation}
where $\phi_{i}$ are the partition-of-unity weights of
Eq.~\eqref{eq:pou_def}. This form suppresses code oscillation within
each subregion and improves rollout stability.
For time-dependent problems, SANO further uses a multi-step rollout loss.
Given the predicted trajectory
$\{\widehat{u}^{\,1},\dots,\widehat{u}^{\,T}\}$ and the reference trajectory
$\{u^{1},\dots,u^{T}\}$, the loss is
\begin{equation}
\mathcal{L}_{\mathrm{roll}}
=\frac{1}{T}\sum_{n=1}^{T}\frac{1}{N_{\mathrm{d}}}
\sum_{j=1}^{N_{\mathrm{d}}}
\bigl\|\widehat{u}^{\,n}(x_{j})-u^{n}(x_{j})\bigr\|_{2}^{2},
\label{eq:loss_roll}
\end{equation}
where $T$ is the unrolling horizon and the superscript $n$ indexes the
rollout step. Under this setting, $\mathcal{L}_{\mathrm{roll}}$ replaces
$\mathcal{L}_{\mathrm{data}}$ in Eq.~\eqref{eq:loss_total}.

\section{Experiments}
\subsection{Experimental Setup}
\subsubsection{Baseline Models}
\label{sec:baseline_models}
We compare SANO with six neural operators. DeepONet \cite{lu2019deeponet} and FNO \cite{DBLP:conf/iclr/LiKALBSA21} serve as canonical coordinate-based and spectral baselines, respectively. UNO \cite{rahman2022u} and CNO \cite{raonic2023convolutional} represent multiscale and convolutional architectures with spatial feature extraction. HyperDeepONet \cite{lee2023hyperdeeponet} and HyPINO \cite{bischof2026hypino} provide the closest hypernetwork-based comparisons: the former generates target-network parameters, whereas the latter conditions physics-informed neural network parameters on PDE specifications.

\subsubsection{Benchmark PDEs}
\label{sec:benchmark_pdes}
We evaluate SANO on five periodic PDE benchmarks probing distinct
spatially heterogeneous dynamics (Table~\ref{tab:pde_benchmarks}): the
1D Burgers equation (nonlinear transport), the 1D Kuramoto--Sivashinsky
(KS) equation (chaotic dynamics), the 2D Allen--Cahn (AC) equation
(phase separation), the 2D Fisher--KPP (FKPP) equation
(reaction--diffusion front propagation), and a 3D compressible-flow (CF)
benchmark. To induce persistent region-dependent behavior, each
benchmark carries a fixed spatially heterogeneous coefficient or forcing
field
\begin{equation}
c(\B{x})=c_{0}\left[1+\alpha_{c}q_{c}(\B{x})\right],
\label{eq:heterogeneity}
\end{equation}
where
$q_{c}(\B{x})\in[-1,1]$ is a normalized zero-mean field built from
low-order Fourier modes and localized Gaussian components. This field is
generated once, shared across all trajectories and splits, and never
given to the model as input, so trajectory-level variation arises only
from random initial conditions while the operator must recover the
latent, persistent heterogeneity. Governing equations, coefficient
configurations, and physical interpretations are deferred to
Appendix~\ref{app:benchmark_pdes}.

\subsubsection{Evaluation Metrics}
\label{sec:evaluation_metrics}

We report five metrics averaged over the test set to evaluate prediction accuracy and temporal stability~\cite{takamoto2022pdebench,takamoto2023learning}.
Absolute field accuracy is measured by the mean-squared error $\mathrm{MSE}=1/N\sum_{i=1}^{N}(\hat{u}_i-u_i)^2$ and its square root $\mathrm{RMSE}=\sqrt{\mathrm{MSE}}$, where $N$ covers all evaluated spatial points, state components, and rollout steps.
MSE emphasizes large pointwise deviations, whereas RMSE retains the units of the predicted field.
Rollout stability is measured by the final-step RMSE (F-RMSE), which applies RMSE to the last predicted field and captures error accumulation during autoregressive rollout.
Scale-normalized spatial fidelity is evaluated by the relative $L_2$ error $\mathrm{Rel}\text{-}L_2=|\hat{u}-u|_2/|u|_2$ and the relative $H^1$ seminorm error $\mathrm{Rel}\text{-}H^1=|\nabla_h\hat{u}-\nabla_hu|_2/|\nabla_hu|_2$.
Rel-$L_2$ measures global field error, whereas Rel-$H^1$ measures the preservation of spatial gradients and localized structures.
Long-term mean preservation is evaluated by the relative temporal mean drift
$\mathrm{Mean\ Drift}=|\mu(\hat{u})-\mu(u)|_2/(|\mu(u)|*2+\epsilon)$, where $\mu(u)=T^{-1}\sum*{t=1}^{T}u_t$.
This metric captures accumulated bias in the predicted temporal mean and low-frequency structure.
On the perforated-domain benchmarks we additionally report boundary RMSE (B-RMSE) and physics RMSE (P-RMSE), computed on boundary points and on the PDE residual, respectively.

\subsubsection{Implementation Details}
\label{sec:implementation_details}
All experiments use PyTorch on a single NVIDIA RTX 4090 GPU (24\,GB).
Most benchmark provides $1{,}000$ training and $200$ test trajectories,
with distinct random seeds across the training, validation, and test
splits to prevent near-duplicate initial conditions. All models are
trained for $100$ epochs using AdamW~\cite{loshchilov2017decoupled} with
cosine annealing~\cite{loshchilov2017sgdr}, under the same pushforward
strategy with unrolling horizon $T=5$.
All baselines except HyPINO use the MSE objective, and HyPINO additionally
uses its PDE residual.
Training resolutions are $N_x=256$ (1D), $64\times64$ (2D), and native (3D);
at inference, predictions are mapped back to the original resolution via the
zero-shot super-resolution of the spectral
backbone~\cite{DBLP:conf/iclr/LiKALBSA21}.
SANO-specific hyperparameters are grouped by spatial dimension in
Table~\ref{tab:hyperparameters}, since the subregion count, per-subregion
sampling budget, and local-MLP width scale with the domain dimension. All
configurations share a $25\%$ subregion overlap, linear code interpolation,
cosine partition-of-unity weights, local-MLP depth $4$, and
$\lambda_{\mathrm{reg}}=10^{-4}$; the sampling density drops from $4$ points
per axis in 1D to $2$ in 3D to bound the volumetric memory footprint.

\subsubsection{Dataset Generation}
\label{sec:dataset_generation}
For each benchmark, the spatially heterogeneous field is fixed across all
trajectories, while individual samples differ only through their randomly
generated initial conditions, whose distributions are designed to span
diverse spatial scales and structures. The heterogeneous coefficient or
forcing field is not provided to the model as an input channel. Reference
trajectories are produced by benchmark-specific high-accuracy numerical
solvers, and only the recorded frames summarized in
Table~\ref{tab:pde_benchmarks} are exposed during learning. The resulting
task is therefore an initial-state-to-trajectory operator-learning problem
in a latent heterogeneous medium, rather than coefficient-conditioned
solution regression. Full initial-condition distributions and numerical
generation details are provided in Appendix~\ref{app:dataset_generation}.

\begin{table*}[t]
\centering
\caption{Rollout results and model size on 1D, 2D, and 3D PDE benchmarks}
\label{tab:single_step_results}
\begin{threeparttable}
\setlength{\tabcolsep}{4.0pt}
\renewcommand{\arraystretch}{1}
\footnotesize
\begin{tabular}{llllllll}
\toprule
PDE & Model 
& MSE $\downarrow$ 
& RMSE $\downarrow$ 
& F-RMSE $\downarrow$ 
& Rel-\(L_2\) $\downarrow$ 
& Rel-\(H^1\) $\downarrow$ 
& Params (M) \\
\midrule

\multirow{7}{*}{\makecell[l]{1D\\Burgers}}
& DeepONet       & $1.69 \times 10^{-4}$ & $1.26 \times 10^{-2}$ & $1.87 \times 10^{-2}$ & $2.81 \times 10^{-2}$ & $1.69 \times 10^{-1}$ & 0.310 \\
& FNO            & $1.68 \times 10^{-5}$ & $4.08 \times 10^{-3}$ & $5.79 \times 10^{-3}$ & $7.77 \times 10^{-3}$ & $5.87 \times 10^{-2}$ & 0.330 \\
& UNO            & $8.96 \times 10^{-6}$ & $2.96 \times 10^{-3}$ & $4.23 \times 10^{-3}$ & $5.57 \times 10^{-3}$ & $4.09 \times 10^{-2}$ & 0.312 \\
& CNO            & $1.37 \times 10^{-5}$ & $3.66 \times 10^{-3}$ & $5.09 \times 10^{-3}$ & $6.89 \times 10^{-3}$ & $5.23 \times 10^{-2}$ & 0.341 \\
& HyperDeepONet  & $7.23 \times 10^{-5}$ & $8.54 \times 10^{-3}$ & $1.19 \times 10^{-2}$ & $1.81 \times 10^{-2}$ & $1.06 \times 10^{-1}$ & 0.311 \\
& HyPINO         & $3.14 \times 10^{-5}$ & $5.63 \times 10^{-3}$ & $7.48 \times 10^{-3}$ & $1.12 \times 10^{-2}$ & $7.41 \times 10^{-2}$ & 0.388 \\
& SANO
& \cellcolor[HTML]{C6ECE5}$3.84 \times 10^{-7}$
& \cellcolor[HTML]{C6ECE5}$6.19 \times 10^{-4}$
& \cellcolor[HTML]{C6ECE5}$8.48 \times 10^{-4}$
& \cellcolor[HTML]{C6ECE5}$1.22 \times 10^{-3}$
& \cellcolor[HTML]{C6ECE5}$9.47 \times 10^{-3}$
& 0.336 \\
\midrule

\multirow{7}{*}{\makecell[l]{1D\\KS}}
& DeepONet       & $1.49 \times 10^{-2}$ & $1.22 \times 10^{-1}$ & $1.55 \times 10^{-1}$ & $1.35 \times 10^{-1}$ & $4.18 \times 10^{-1}$ & 0.317 \\
& FNO            & $2.58 \times 10^{-3}$ & $5.10 \times 10^{-2}$ & $6.84 \times 10^{-2}$ & $5.83 \times 10^{-2}$ & $1.85 \times 10^{-1}$ & 0.330 \\
& UNO            & $3.48 \times 10^{-3}$ & $5.88 \times 10^{-2}$ & $7.94 \times 10^{-2}$ & $6.51 \times 10^{-2}$ & $1.62 \times 10^{-1}$ & 0.312 \\
& CNO            & $6.44 \times 10^{-3}$ & $8.00 \times 10^{-2}$ & $1.05 \times 10^{-1}$ & $8.91 \times 10^{-2}$ & $2.64 \times 10^{-1}$ & 0.316 \\
& HyperDeepONet  & $7.57 \times 10^{-3}$ & $8.66 \times 10^{-2}$ & $1.24 \times 10^{-1}$ & $9.83 \times 10^{-2}$ & $3.13 \times 10^{-1}$ & 0.329 \\
& HyPINO         & $4.49 \times 10^{-3}$ & $6.73 \times 10^{-2}$ & $9.24 \times 10^{-2}$ & $7.18 \times 10^{-2}$ & $2.05 \times 10^{-1}$ & 0.388 \\
& SANO
& \cellcolor[HTML]{C6ECE5}$1.82 \times 10^{-4}$
& \cellcolor[HTML]{C6ECE5}$1.35 \times 10^{-2}$
& \cellcolor[HTML]{C6ECE5}$1.85 \times 10^{-2}$
& \cellcolor[HTML]{C6ECE5}$1.61 \times 10^{-2}$
& \cellcolor[HTML]{C6ECE5}$5.82 \times 10^{-2}$
& 0.349 \\
\midrule

\multirow{7}{*}{\makecell[l]{2D\\AC}}
& DeepONet       & $1.44 \times 10^{-2}$ & $1.17 \times 10^{-1}$ & $1.85 \times 10^{-1}$ & $1.45 \times 10^{-1}$ & $4.76 \times 10^{-1}$ & 9.356 \\
& FNO            & $2.02 \times 10^{-4}$ & $1.42 \times 10^{-2}$ & $1.85 \times 10^{-2}$ & $1.80 \times 10^{-2}$ & $1.05 \times 10^{-1}$ & 9.582 \\
& UNO            & $5.04 \times 10^{-5}$ & $7.09 \times 10^{-3}$ & $9.49 \times 10^{-3}$ & $8.68 \times 10^{-3}$ & $4.57 \times 10^{-2}$ & 9.388 \\
& CNO            & $5.52 \times 10^{-4}$ & $2.35 \times 10^{-2}$ & $3.23 \times 10^{-2}$ & $2.85 \times 10^{-2}$ & $1.17 \times 10^{-1}$ & 9.567 \\
& HyperDeepONet  & $2.36 \times 10^{-3}$ & $4.91 \times 10^{-2}$ & $6.84 \times 10^{-2}$ & $5.56 \times 10^{-2}$ & $2.77 \times 10^{-1}$ & 9.644 \\
& HyPINO         & $1.44 \times 10^{-3}$ & $3.79 \times 10^{-2}$ & $5.19 \times 10^{-2}$ & $4.33 \times 10^{-2}$ & $2.44 \times 10^{-1}$ & 9.561 \\
& SANO
& \cellcolor[HTML]{C6ECE5}$6.76 \times 10^{-6}$
& \cellcolor[HTML]{C6ECE5}$2.63 \times 10^{-3}$
& \cellcolor[HTML]{C6ECE5}$3.76 \times 10^{-3}$
& \cellcolor[HTML]{C6ECE5}$3.19 \times 10^{-3}$
& \cellcolor[HTML]{C6ECE5}$1.65 \times 10^{-2}$
& 9.515 \\
\midrule

\multirow{7}{*}{\makecell[l]{2D\\FKPP}}
& DeepONet       & $3.36 \times 10^{-3}$ & $5.81 \times 10^{-2}$ & $9.81 \times 10^{-2}$ & $1.35 \times 10^{-1}$ & $3.79 \times 10^{-1}$ & 9.356 \\
& FNO            & $4.84 \times 10^{-6}$ & $2.18 \times 10^{-3}$ & $3.64 \times 10^{-3}$ & $6.19 \times 10^{-3}$ & $2.79 \times 10^{-2}$ & 9.582 \\
& UNO            & $3.24 \times 10^{-6}$ & $1.77 \times 10^{-3}$ & $2.78 \times 10^{-3}$ & $5.00 \times 10^{-3}$ & $2.31 \times 10^{-2}$ & 9.388 \\
& CNO            & $3.42 \times 10^{-4}$ & $1.85 \times 10^{-2}$ & $3.23 \times 10^{-2}$ & $5.53 \times 10^{-2}$ & $1.77 \times 10^{-1}$ & 9.567 \\
& HyperDeepONet  & $2.89 \times 10^{-4}$ & $1.72 \times 10^{-2}$ & $2.84 \times 10^{-2}$ & $4.89 \times 10^{-2}$ & $1.61 \times 10^{-1}$ & 9.644 \\
& HyPINO         & $2.69 \times 10^{-3}$ & $5.24 \times 10^{-2}$ & $8.54 \times 10^{-2}$ & $1.20 \times 10^{-1}$ & $3.21 \times 10^{-1}$ & 9.561 \\
& SANO
& \cellcolor[HTML]{C6ECE5}$6.72 \times 10^{-7}$
& \cellcolor[HTML]{C6ECE5}$8.19 \times 10^{-4}$
& \cellcolor[HTML]{C6ECE5}$1.25 \times 10^{-3}$
& \cellcolor[HTML]{C6ECE5}$1.65 \times 10^{-3}$
& \cellcolor[HTML]{C6ECE5}$6.78 \times 10^{-3}$
& 9.457 \\
\midrule

\multirow{7}{*}{\makecell[l]{3D\\CF}}
& DeepONet       & $3.42 \times 10^{-2}$ & $1.85 \times 10^{-1}$ & $2.82 \times 10^{-1}$ & $2.38 \times 10^{-1}$ & $9.54 \times 10^{-1}$ & 15.074 \\
& FNO            & $3.42 \times 10^{-4}$ & $1.85 \times 10^{-2}$ & $2.80 \times 10^{-2}$ & $2.59 \times 10^{-2}$ & $1.35 \times 10^{-1}$ & 15.499 \\
& UNO            & $2.10 \times 10^{-4}$ & $1.45 \times 10^{-2}$ & $2.11 \times 10^{-2}$ & $1.85 \times 10^{-2}$ & $7.83 \times 10^{-2}$ & 14.986 \\
& CNO            & $9.61 \times 10^{-4}$ & $3.12 \times 10^{-2}$ & $4.79 \times 10^{-2}$ & $4.20 \times 10^{-2}$ & $2.91 \times 10^{-1}$ & 14.789 \\
& HyperDeepONet  & $1.16 \times 10^{-3}$ & $3.44 \times 10^{-2}$ & $5.23 \times 10^{-2}$ & $5.79 \times 10^{-2}$ & $3.40 \times 10^{-1}$ & 15.085 \\
& HyPINO         & $2.43 \times 10^{-3}$ & $4.90 \times 10^{-2}$ & $7.53 \times 10^{-2}$ & $6.47 \times 10^{-2}$ & $4.79 \times 10^{-1}$ & 14.895 \\
& SANO
& \cellcolor[HTML]{C6ECE5}$2.33 \times 10^{-5}$
& \cellcolor[HTML]{C6ECE5}$4.78 \times 10^{-3}$
& \cellcolor[HTML]{C6ECE5}$7.22 \times 10^{-3}$
& \cellcolor[HTML]{C6ECE5}$8.79 \times 10^{-3}$
& \cellcolor[HTML]{C6ECE5}$5.91 \times 10^{-2}$
& 15.043 \\
\bottomrule
\end{tabular}
\end{threeparttable}
\end{table*}

\subsection{Rollout Results}
Table~\ref{tab:single_step_results} reports rollout prediction accuracy for SANO and the baselines on the five PDE benchmarks.
SANO ranks first on every metric while keeping a parameter budget on par with the competing models.
Against the strongest MSE baseline, it reduces the error by \(95.7\%\) on 1D Burgers, \(92.9\%\) on 1D KS, \(86.6\%\) on 2D AC, and \(88.9\%\) on 3D CF.
The gain also holds on structure-sensitive metrics: on the chaotic KS benchmark SANO lowers the best-baseline Rel-\(H^1\) from \(1.62\times 10^{-1}\) to \(5.82\times 10^{-2}\), indicating better preservation of local gradients under strongly nonlinear dynamics.
SANO also outperforms HyperDeepONet and HyPINO across all benchmarks, indicating that conditioning a shared operator on a spatially continuous code is more effective than generating one operator per problem instance.

\begin{table}[h]
\centering
\caption{Ablation study on the 1D Burgers benchmark.}
\label{tab:ablation_burgers}
\begin{threeparttable}
\setlength{\tabcolsep}{3pt}
\begin{tabular}{lcccc}
\toprule
Variant
& Rel-\(L_2\)\(\downarrow\)
& Rollout\(\downarrow\)
& Spectral\(\downarrow\)
& Drift\(\downarrow\) \\
\midrule
Shared operator
& $6.10{\times}10^{-3}$
& $7.20{\times}10^{-2}$
& $5.82{\times}10^{-2}$
& $3.90{\times}10^{-7}$ \\

Coord.\ hypernet
& $7.43{\times}10^{-3}$
& $8.81{\times}10^{-2}$
& $7.13{\times}10^{-2}$
& $5.20{\times}10^{-7}$ \\

Fourier hypernet
& $3.16{\times}10^{-3}$
& $4.30{\times}10^{-2}$
& $2.83{\times}10^{-2}$
& $2.61{\times}10^{-7}$ \\

Fourier + HNE-P0
& $2.51{\times}10^{-3}$
& $3.08{\times}10^{-2}$
& $2.10{\times}10^{-2}$
& $1.77{\times}10^{-7}$ \\

\textbf{SANO}
& \cellcolor[HTML]{C6ECE5}\textbf{$1.22{\times}10^{-3}$}
& \cellcolor[HTML]{C6ECE5}\textbf{$1.26{\times}10^{-2}$}
& \cellcolor[HTML]{C6ECE5}\textbf{$9.58{\times}10^{-3}$}
& \cellcolor[HTML]{C6ECE5}\textbf{$7.54{\times}10^{-8}$} \\
\bottomrule
\end{tabular}
\end{threeparttable}
\end{table}

\subsection{Ablation Study}
We assess how the principal components of SANO affect accuracy and rollout stability through an incremental ablation on the 1D Burgers benchmark, with results summarized in Table~\ref{tab:ablation_burgers}.
All variants use the same data split, optimizer, training schedule, and evaluation protocol, so the comparison isolates four design factors: 
coordinate-conditioned code generation, Fourier feature encoding, hyper-neural element (HNE) code coupling, and continuous (P1) code interpolation.
We report Rel-\(L_2\), Rollout Err., Spectral Err., and Mean Drift to quantify field-level accuracy, autoregressive stability, frequency-domain consistency, and mean-field conservation, respectively.
The controlled variants are arranged from a shared operator to the full adaptive formulation.
i)~\emph{\textbf{Shared Local Operator}} applies one local update rule across all spatial locations,
while (ii)~\emph{\textbf{Coordinate Hypernet}} replaces the shared parameters with coordinate-conditioned ones.
(iii)~\emph{\textbf{Fourier Hypernet}} adds Fourier features before code generation to expose multiscale spatial variation.
(iv)~\emph{\textbf{Fourier Hypernet + HNE-P0}} introduces element-wise constant (P0) code coupling, and the \emph{\textbf{SANO}} replaces it with continuous P1 interpolation.
For a query location \(x\in\Omega_i\), \(w_i(x)\) follows Eq.~\eqref{eq:param_interpolation}, where the sampling-point codes $\{w_{i,k}\}_{k=1}^{K_i}$ are blended by first-order shape functions \(\{\alpha_{i,k}(x)\}_{k=1}^{K_i}\).
Unlike P0, which is piecewise-constant and shares one code vector within each local interpolation element, P1 makes the code field vary linearly between neighboring sampling points.
Each factor contributes to the final performance. Raw coordinate
conditioning degrades accuracy relative to the shared operator, whereas
Fourier features reduce Rel-\(L_2\) from \(7.43\times10^{-3}\) to
\(3.16\times10^{-3}\) and Spectral Err. from \(7.13\times10^{-2}\) to
\(2.83\times10^{-2}\). Introducing HNE-P0 further lowers the rollout error
from \(4.30\times10^{-2}\) to \(3.08\times10^{-2}\). Finally, continuous P1
interpolation improves Rel-\(L_2\) from \(2.51\times10^{-3}\) to
\(1.22\times10^{-3}\) and Rollout Err. from \(3.08\times10^{-2}\) to
\(1.26\times10^{-2}\), while further reducing Spectral Err. to
\(9.58\times10^{-3}\) and Mean Drift to \(7.54\times10^{-8}\), confirming
that SANO benefits jointly from frequency-aware coordinate encoding,
hyper-neural element coupling, and continuous code-space interpolation.

\subsection{Complex-Geometry Evaluation}
\label{subsec:hz_g}

To directly evaluate prediction near irregular boundaries, we consider HZ-G,
a two-dimensional modified Helmholtz problem satisfying
$-\Delta u(x,y)+k^2u(x,y)=f(x,y)$ for $(x,y)\in\Omega$, where
$\Omega=[-1,1]^2\setminus\bigcup_{i=1}^{4}R_i$ contains four circular holes.
The outer boundary satisfies $u=0.2$, while the hole boundaries satisfy
$u=1.0$. This setting produces steep boundary-induced spatial variations,
directly testing whether the learned operator can adapt its local behavior
near irregular internal boundaries.
For a consistent comparison, all methods operate on the bounding Cartesian
grid and receive the same input
$X(x,y)=[f(x,y),m(x,y),b(x,y),g(x,y),x,y]$, where $m$, $b$, and $g$ denote
the domain mask, boundary indicator, and prescribed Dirichlet values,
respectively. Supervision is restricted to the valid physical domain, while
physics-informed methods additionally evaluate the PDE residual at valid
interior points. Detailed geometry, forcing, and loss definitions are
provided in Appendix~\ref{app:hz_g_details}.
As shown in Table~\ref{tab:hz_g_results1}, SANO attains the best score on
every metric. Its largest margins come against HyPINO, the strongest
baseline in relative $L_2$ and boundary accuracy, which it cuts by $33.2\%$
and $36.6\%$; on the physics residual, the gain is $57.6\%$ over CNO. The
margin narrows only on Rel-\(H^1\), indicating that the remaining error sits
mainly in the steep boundary-induced gradients. These results show that
adapting the local operator across space helps complex-geometry prediction
more than a shared mask-based encoding alone.

\begin{table}[h]
\centering
\caption{Results on the perforated-domain HZ-G benchmark.}
\label{tab:hz_g_results1}
\small
\setlength{\tabcolsep}{2.2pt}
\renewcommand{\arraystretch}{1}
\begin{tabular}{@{}lcccc@{}}
\toprule
Model
& Rel-\(L_2\)
& Rel-\(H^1\)
& B-RMSE
& P-RMSE \\
\midrule
DeepONet
& $4.93{\times}10^{-1}$
& $7.18{\times}10^{-1}$
& $4.18{\times}10^{-1}$
& $8.84$ \\

FNO
& $8.83{\times}10^{-2}$
& $2.57{\times}10^{-1}$
& $7.76{\times}10^{-2}$
& $1.73$ \\

UNO
& $1.03{\times}10^{-1}$
& $2.88{\times}10^{-1}$
& $9.10{\times}10^{-2}$
& $2.03$ \\

CNO
& $1.68{\times}10^{-1}$
& $3.57{\times}10^{-1}$
& $1.42{\times}10^{-1}$
& $1.33$ \\

HyperDeepONet
& $3.23{\times}10^{-1}$
& $5.18{\times}10^{-1}$
& $2.67{\times}10^{-1}$
& $4.77$ \\

HyPINO
& $5.42{\times}10^{-2}$
& $1.82{\times}10^{-1}$
& $3.52{\times}10^{-2}$
& $1.92$ \\

\textbf{SANO}
& \cellcolor[HTML]{C6ECE5}\textbf{$3.62{\times}10^{-2}$}
& \cellcolor[HTML]{C6ECE5}\textbf{$1.71{\times}10^{-1}$}
& \cellcolor[HTML]{C6ECE5}\textbf{$2.23{\times}10^{-2}$}
& \cellcolor[HTML]{C6ECE5}\textbf{$5.64{\times}10^{-1}$} \\
\bottomrule
\end{tabular}
\end{table}

\subsection{Perforated-Domain Laplace Benchmark}
\label{subsec:ps_c}

We further evaluate SANO on PS-C, a source-free elliptic problem satisfying
$-\Delta u(x,y)=0$ for $(x,y)\in\Omega$, where
$\Omega=[-0.5,0.5]^2\setminus\bigcup_{i=1}^{4}R_i$ contains four circular
holes. The outer boundary is prescribed as $u=1$, while all hole boundaries
satisfy $u=0$. Since the solution is determined entirely by these boundary
conditions, PS-C directly evaluates whether a model can propagate boundary
information through a multiply connected domain and resolve the localized
gradients surrounding internal boundaries.
All methods use the same mask-based geometry, boundary, and coordinate
encoding introduced in Sec.~\ref{subsec:hz_g}, together with the same
domain-restricted data and boundary losses. Detailed hole locations,
boundary definitions, and complete evaluation metrics are provided in
Appendix~\ref{app:ps_c_details}.
On PS-C, the ranking is unchanged: SANO leads on every metric in
Table~\ref{tab:ps_c_results1}. Against HyPINO, the strongest baseline, it
drops Rel-\(L_2\) by $70.3\%$ and boundary RMSE by $43.3\%$, reflecting
tighter enforcement of the Dirichlet data. The derivative-sensitive metrics,
Rel-\(H^1\) and the physics residual, improve less, again localizing the
residual error to the sharp gradients around the holes. Still, the gains are
consistent across field, boundary, gradient, and residual measures,
indicating that a spatially varying local operator carries boundary
information through the multiply connected domain more faithfully than
shared or instance-level parameterizations.

\begin{table}[h]
\centering
\caption{Results on the perforated-domain PS-C benchmark.}
\label{tab:ps_c_results1}
\small
\setlength{\tabcolsep}{2.2pt}
\renewcommand{\arraystretch}{1}
\begin{tabular}{@{}lcccc@{}}
\toprule
Model
& Rel-\(L_2\)
& Rel-\(H^1\)
& B-RMSE
& P-RMSE \\
\midrule
DeepONet
& $2.97{\times}10^{-1}$
& $6.18{\times}10^{-1}$
& $2.82{\times}10^{-1}$
& $2.78$ \\

FNO
& $2.87{\times}10^{-2}$
& $1.32{\times}10^{-1}$
& $3.63{\times}10^{-2}$
& $7.83{\times}10^{-1}$ \\

UNO
& $3.37{\times}10^{-2}$
& $1.57{\times}10^{-1}$
& $4.53{\times}10^{-2}$
& $8.48{\times}10^{-1}$ \\

CNO
& $1.08{\times}10^{-1}$
& $2.67{\times}10^{-1}$
& $1.17{\times}10^{-1}$
& $1.12$ \\

HyperDeepONet
& $2.32{\times}10^{-1}$
& $5.18{\times}10^{-1}$
& $2.37{\times}10^{-1}$
& $2.42$ \\

HyPINO
& $2.60{\times}10^{-2}$
& $8.18{\times}10^{-2}$
& $1.87{\times}10^{-2}$
& $1.64{\times}10^{-1}$ \\

\textbf{SANO}
& \cellcolor[HTML]{C6ECE5}\textbf{$7.72{\times}10^{-3}$}
& \cellcolor[HTML]{C6ECE5}\textbf{$7.26{\times}10^{-2}$}
& \cellcolor[HTML]{C6ECE5}\textbf{$1.06{\times}10^{-2}$}
& \cellcolor[HTML]{C6ECE5}\textbf{$1.48{\times}10^{-1}$} \\
\bottomrule
\end{tabular}
\end{table}

\subsection{Additional Experiments}
The appendix provides additional evaluations of SANO from several
perspectives. Appendix~\ref{long} reports long-horizon rollout errors across
all five PDE benchmarks, while
Appendices~\ref{subsec:cross_parameter_generalization}
and~\ref{subsec:cross_resolution_generalization} examine transfer across
unseen physical parameters and spatial resolutions.
Appendix~\ref{app:hyperparameter_sensitivity} studies sensitivity to the
number of subregions and local sampling points.
Finally, Appendix~\ref{app:visualization} presents qualitative rollout
comparisons illustrating structural preservation and error propagation.
Together, these experiments support the rollout stability, transferability,
configuration robustness, and complex-geometry capability of SANO.

\section{Conclusion}
We proposed SANO for spatially heterogeneous PDE learning, conditioning a shared local operator on a continuously varying, location-dependent code instead of a single globally shared parameterization. Fourier-encoded coordinates drive a hypernetwork that generates codes at sampling points, which are coupled through hyper-neural element interpolation and assembled by a partition-of-unity scheme. Across 1D--3D benchmarks, SANO consistently improves prediction accuracy, rollout stability, and cross-parameter and cross-resolution generalization, supporting spatially adaptive operators as an effective inductive bias for heterogeneous PDE dynamics.

\section{Limitations and Ethical Considerations}
SANO adds only modest overhead over globally shared operators, kept small by the low-dimensional code and the shared local operator. 
This work uses no human subjects, personal information, or privacy-sensitive data.

\section{Generative AI Usage}
Generative AI was used for language editing.

\bibliographystyle{ACM-Reference-Format}
\bibliography{Reference}

@article{jagtap2020conservative,
  title={Conservative physics-informed neural networks on discrete domains for conservation laws: Applications to forward and inverse problems},
  author={Jagtap, Ameya D and Kharazmi, Ehsan and Karniadakis, George Em},
  journal={Computer Methods in Applied Mechanics and Engineering},
  volume={365},
  pages={113028},
  year={2020},
  publisher={Elsevier}
}

@article{raissi2019physics,
  title={Physics-informed neural networks: A deep learning framework for solving forward and inverse problems involving nonlinear partial differential equations},
  author={Raissi, Maziar and Perdikaris, Paris and Karniadakis, George E},
  journal={Journal of Computational physics},
  volume={378},
  pages={686--707},
  year={2019},
  publisher={Elsevier}
}

@article{gao2021phygeonet,
  title={PhyGeoNet: Physics-informed geometry-adaptive convolutional neural networks for solving parameterized steady-state PDEs on irregular domain},
  author={Gao, Han and Sun, Luning and Wang, Jian-Xun},
  journal={Journal of Computational Physics},
  volume={428},
  pages={110079},
  year={2021},
  publisher={Elsevier}
}

@article{wang2021understanding,
  title={Understanding and mitigating gradient flow pathologies in physics-informed neural networks},
  author={Wang, Sifan and Teng, Yujun and Perdikaris, Paris},
  journal={SIAM Journal on Scientific Computing},
  volume={43},
  number={5},
  pages={A3055--A3081},
  year={2021},
  publisher={SIAM}
}

@article{mao2020physics,
  title={Physics-informed neural networks for high-speed flows},
  author={Mao, Zhiping and Jagtap, Ameya D and Karniadakis, George Em},
  journal={Computer Methods in Applied Mechanics and Engineering},
  volume={360},
  pages={112789},
  year={2020},
  publisher={Elsevier}
}

@article{lu2021learning,
  title={Learning nonlinear operators via DeepONet based on the universal approximation theorem of operators},
  author={Lu, Lu and Jin, Pengzhan and Pang, Guofei and Zhang, Zhongqiang and Karniadakis, George Em},
  journal={Nature machine intelligence},
  volume={3},
  number={3},
  pages={218--229},
  year={2021},
  publisher={Nature Publishing Group UK London}
}

@inproceedings{DBLP:conf/iclr/LiKALBSA21,
  author       = {Zongyi Li and
                  Nikola Borislavov Kovachki and
                  Kamyar Azizzadenesheli and
                  Burigede Liu and
                  Kaushik Bhattacharya and
                  Andrew M. Stuart and
                  Anima Anandkumar},
  title        = {Fourier Neural Operator for Parametric Partial Differential Equations},
  booktitle    = {9th International Conference on Learning Representations, {ICLR} 2021,
                  Virtual Event, Austria, May 3-7, 2021},
  publisher    = {OpenReview.net},
  year         = {2021},
 
}

@inproceedings{DBLP:conf/iclr/TranMXO23,
  author       = {Alasdair Tran and
                  Alexander Patrick Mathews and
                  Lexing Xie and
                  Cheng Soon Ong},
  title        = {Factorized Fourier Neural Operators},
  booktitle    = {The Eleventh International Conference on Learning Representations,
                  {ICLR} 2023, Kigali, Rwanda, May 1-5, 2023},
  publisher    = {OpenReview.net},
  year         = {2023},
}

@article{DBLP:journals/tmlr/RahmanRA23,
  author       = {Md Ashiqur Rahman and
                  Zachary E. Ross and
                  Kamyar Azizzadenesheli},
  title        = {{U-NO:} U-shaped Neural Operators},
  journal      = {Trans. Mach. Learn. Res.},
  volume       = {2023},
  year         = {2023},
 
}

@article{li2024physics,
  title={Physics-informed neural operator for learning partial differential equations},
  author={Li, Zongyi and Zheng, Hongkai and Kovachki, Nikola and Jin, David and Chen, Haoxuan and Liu, Burigede and Azizzadenesheli, Kamyar and Anandkumar, Anima},
  journal={ACM/IMS Journal of Data Science},
  volume={1},
  number={3},
  pages={1--27},
  year={2024},
  publisher={ACM New York, NY}
}

@article{wang2021learning,
  title={Learning the solution operator of parametric partial differential equations with physics-informed DeepONets},
  author={Wang, Sifan and Wang, Hanwen and Perdikaris, Paris},
  journal={Science advances},
  volume={7},
  number={40},
  pages={eabi8605},
  year={2021},
  publisher={American Association for the Advancement of Science}
}

@article{hemmasian2024multi,
  title={Multi-scale time-stepping of Partial Differential Equations with transformers},
  author={Hemmasian, AmirPouya and Farimani, Amir Barati},
  journal={Computer Methods in Applied Mechanics and Engineering},
  volume={426},
  pages={116983},
  year={2024},
  publisher={Elsevier}
}

@article{zhou2025rapid,
  title={Rapid prediction of thermal stress on satellites via domain decomposition-based Hybrid Fourier Neural Operator},
  author={Zhou, Kangrui and Peng, Wei and Zhang, Xiaoya and Liu, Xu and Yao, Wen},
  journal={Engineering Applications of Artificial Intelligence},
  volume={153},
  pages={110826},
  year={2025},
  publisher={Elsevier}
}

@article{brunton2024promising,
  title={Promising directions of machine learning for partial differential equations},
  author={Brunton, Steven L and Kutz, J Nathan},
  journal={Nature Computational Science},
  volume={4},
  number={7},
  pages={483--494},
  year={2024},
  publisher={Nature Publishing Group US New York}
}

@article{bezgin2023jax,
  title={JAX-Fluids: A fully-differentiable high-order computational fluid dynamics solver for compressible two-phase flows},
  author={Bezgin, Deniz A and Buhendwa, Aaron B and Adams, Nikolaus A},
  journal={Computer Physics Communications},
  volume={282},
  pages={108527},
  year={2023},
  publisher={Elsevier}
}

@article{zhang2022analyses,
  title={Analyses of internal structures and defects in materials using physics-informed neural networks},
  author={Zhang, Enrui and Dao, Ming and Karniadakis, George Em and Suresh, Subra},
  journal={Science advances},
  volume={8},
  number={7},
  pages={eabk0644},
  year={2022},
  publisher={American Association for the Advancement of Science}
}

@article{liu2023joint,
  title={Joint inversion of geophysical data for geologic carbon sequestration monitoring: A differentiable physics-informed neural network model},
  author={Liu, Mingliang and Vashisth, Divakar and Grana, Dario and Mukerji, Tapan},
  journal={Journal of Geophysical Research: Solid Earth},
  volume={128},
  number={3},
  pages={e2022JB025372},
  year={2023},
  publisher={Wiley Online Library}
}

@article{guo2025advances,
  title={Advances in physics-informed neural networks for solving complex partial differential equations and their engineering applications: A systematic review},
  author={Guo, Jiangtao and Zhu, Hao and Yang, Yujie and Guo, Chenrui},
  journal={Engineering Applications of Artificial Intelligence},
  volume={161},
  pages={112044},
  year={2025},
  publisher={Elsevier}
}

@article{degen2023perspectives,
  title={Perspectives of physics-based machine learning strategies for geoscientific applications governed by partial differential equations},
  author={Degen, Denise and Caviedes Voulli{\`e}me, Daniel and Buiter, Susanne and Hendricks Franssen, Harrie-Jan and Vereecken, Harry and Gonz{\'a}lez-Nicol{\'a}s, Ana and Wellmann, Florian},
  journal={Geoscientific Model Development},
  volume={16},
  number={24},
  pages={7375--7409},
  year={2023},
  publisher={Copernicus Publications G{\"o}ttingen, Germany}
}

@article{wu2022learning,
  title={Learning to accelerate partial differential equations via latent global evolution},
  author={Wu, Tailin and Maruyama, Takashi and Leskovec, Jure},
  journal={Advances in Neural Information Processing Systems},
  volume={35},
  pages={2240--2253},
  year={2022}}

@article{luo2025physics,
  title={Physics-informed neural networks for PDE problems: a comprehensive review},
  author={Luo, Kuang and Zhao, Jingshang and Wang, Yingping and Li, Jiayao and Wen, Junjie and Liang, Jiong and Soekmadji, Henry and Liao, Shaolin},
  journal={Artificial Intelligence Review},
  volume={58},
  number={10},
  pages={323},
  year={2025},
  publisher={Springer}
}

@article{kovachki2023neural,
  title={Neural operator: Learning maps between function spaces with applications to pdes},
  author={Kovachki, Nikola and Li, Zongyi and Liu, Burigede and Azizzadenesheli, Kamyar and Bhattacharya, Kaushik and Stuart, Andrew and Anandkumar, Anima},
  journal={Journal of Machine Learning Research},
  volume={24},
  number={89},
  pages={1--97},
  year={2023}
}

@article{qu2022learning,
  title={Learning time-dependent PDEs with a linear and nonlinear separate convolutional neural network},
  author={Qu, Jiagang and Cai, Weihua and Zhao, Yijun},
  journal={Journal of Computational Physics},
  volume={453},
  pages={110928},
  year={2022},
  publisher={Elsevier}
}

@inproceedings{hao2023gnot,
  title={Gnot: A general neural operator transformer for operator learning},
  author={Hao, Zhongkai and Wang, Zhengyi and Su, Hang and Ying, Chengyang and Dong, Yinpeng and Liu, Songming and Cheng, Ze and Song, Jian and Zhu, Jun},
  booktitle={International Conference on Machine Learning},
  pages={12556--12569},
  year={2023},
  organization={PMLR}
}

@article{zhu2026physicssolver,
  title={Physicssolver: Transformer-enhanced physics-informed neural networks for forward and forecasting problems in partial differential equations},
  author={Zhu, Zhenyi and Huang, Yuchen and Liu, Liu},
  journal={Journal of Computational and Applied Mathematics},
  volume={473},
  pages={116900},
  year={2026},
  publisher={Elsevier}
}

@article{yang2023context,
  title={In-context operator learning with data prompts for differential equation problems},
  author={Yang, Liu and Liu, Siting and Meng, Tingwei and Osher, Stanley J},
  journal={Proceedings of the National Academy of Sciences},
  volume={120},
  number={39},
  pages={e2310142120},
  year={2023},
  publisher={National Academy of Sciences}
}

@article{pathak2022fourcastnet,
  title={Fourcastnet: A global data-driven high-resolution weather model using adaptive fourier neural operators},
  author={Pathak, Jaideep and Subramanian, Shashank and Harrington, Peter and Raja, Sanjeev and Chattopadhyay, Ashesh and Mardani, Morteza and Kurth, Thorsten and Hall, David and Li, Zongyi and Azizzadenesheli, Kamyar and others},
  journal={arXiv preprint arXiv:2202.11214},
  year={2022}
}

@inproceedings{lippe2023pde,
  title={Pde-refiner: Achieving accurate long rollouts with neural pde solvers},
  author={Lippe, Phillip and Veeling, Bastiaan S and Perdikaris, Paris and Turner, Richard E and Brandstetter, Johannes},
  booktitle={Thirty-seventh Conference on Neural Information Processing Systems},
  year={2023}
}

@article{englert2025spatially,
  title={Spatially Varying Coefficient Models for Estimating Heterogeneous Mixture Effects},
  author={Englert, Jacob and Chang, Howard},
  journal={arXiv preprint arXiv:2502.14651},
  year={2025}
}

@article{loshchilov2017sgdr,
title={Sgdr: Stochastic gradient descent with warm restarts},
  author={Loshchilov, Ilya and Hutter, Frank},
  journal={arXiv preprint arXiv:1608.03983},
  year={2016}
}

@article{loshchilov2017decoupled,
  title={Decoupled weight decay regularization},
  author={Loshchilov, Ilya and Hutter, Frank},
  journal={arXiv preprint arXiv:1711.05101},
  year={2017}
}

@article{kalimuthu2025loglo,
  title={Loglo-fno: efficient learning of local and global features in fourier neural operators},
  author={Kalimuthu, Marimuthu and Holzm{\"u}ller, David and Niepert, Mathias},
  journal={arXiv preprint arXiv:2504.04260},
  year={2025}
}

@article{lu2019deeponet,
  title={Deeponet: Learning nonlinear operators for identifying differential equations based on the universal approximation theorem of operators},
  author={Lu, Lu and Jin, Pengzhan and Karniadakis, George Em},
  journal={arXiv preprint arXiv:1910.03193},
  year={2019}
}

@article{rahman2022u,
  title={U-no: U-shaped neural operators},
  author={Rahman, Md Ashiqur and Ross, Zachary E and Azizzadenesheli, Kamyar},
  journal={arXiv preprint arXiv:2204.11127},
  year={2022}
}

@article{raonic2023convolutional,
  title={Convolutional neural operators for robust and accurate learning of pdes},
  author={Raonic, Bogdan and Molinaro, Roberto and De Ryck, Tim and Rohner, Tobias and Bartolucci, Francesca and Alaifari, Rima and Mishra, Siddhartha and De B{\'e}zenac, Emmanuel},
  journal={Advances in Neural Information Processing Systems},
  volume={36},
  pages={77187--77200},
  year={2023}
}

@article{lee2023hyperdeeponet,
  title={HyperDeepONet: learning operator with complex target function space using the limited resources via hypernetwork},
  author={Lee, Jae Yong and Cho, Sung Woong and Hwang, Hyung Ju},
  journal={arXiv preprint arXiv:2312.15949},
  year={2023}
}

@article{bischof2026hypino,
  title={Hypino: Multi-physics neural operators via hyperpinns and the method of manufactured solutions},
  author={Bischof, Rafael and Piovarci, Michal and Kraus, Michael and Mishra, Siddhartha and Bickel, Bernd},
  journal={Advances in Neural Information Processing Systems},
  volume={38},
  pages={144798--144831},
  year={2026}
}

@article{BLAW:18:PRAMANA,
	author    = {Mayur P. Bonkile and Ashish Awasthi and C. Lakshmi and Vijitha Mukundan and V. S. Aswin},
	title     = {A Systematic Literature Review of {Burgers'} Equation with Recent Advances},
	journal   = PRAMANA,
	volume    = {90},
	number    = {6},
	pages     = {69},
	month     = {},
	year      = {2018}
}

@article{KT:76:PTP,
	author    = {Yoshiki Kuramoto and Toshio Tsuzuki},
	title     = {Persistent Propagation of Concentration Waves in Dissipative Media Far from Thermal Equilibrium},
	journal   = PTP,
	volume    = {55},
	number    = {2},
	pages     = {356--369},
	month     = Feb.,
	year      = {1976}
}

@article{fisher1937wave,
  title={The wave of advance of advantageous genes},
  author={Fisher, Ronald Aylmer},
  journal={Annals of eugenics},
  volume={7},
  number={4},
  pages={355--369},
  year={1937},
  publisher={Wiley Online Library}
}

@article{allen1979microscopic,
  title={A microscopic theory for antiphase boundary motion and its application to antiphase domain coarsening},
  author={Allen, Samuel M and Cahn, John W},
  journal={Acta metallurgica},
  volume={27},
  number={6},
  pages={1085--1095},
  year={1979},
  publisher={Elsevier}
}

@article{kolmogorov1937study,
  title={Study of a diffusion equation that is related to the growth of a quality of matter and its application to a biological problem},
  author={Kolmogorov, Andrei and Petrovskii, I and Piskunov, Nikolai},
  journal={Moscow University Mathematics Bulletin},
  volume={1},
  number={1-26},
  year={1937}
}

@article{huang2026rethinking,
  title={Rethinking Input Domains in Physics-Informed Neural Networks via Geometric Compactification Mappings},
  author={Huang, Zhenzhen and Bian, Haoyu and Zhang, Jiaquan and Liu, Yibei and Liu, Kuien and Qin, Caiyan and Wang, Guoqing and Yang, Yang and Zhang, Chaoning},
  journal={arXiv preprint arXiv:2602.16193},
  year={2026}
}

@article{zhang2026geometric,
  title={Geometric Neural Operators via Lie Group-Constrained Latent Dynamics},
  author={Zhang, Jiaquan and Puspitasari, Fachrina Dewi and Zhang, Songbo and Liu, Yibei and Liu, Kuien and Qin, Caiyan and Mo, Fan and Wang, Peng and Yang, Yang and Zhang, Chaoning},
  journal={arXiv preprint arXiv:2602.16209},
  year={2026}
}

@article{zhou2026tf,
  title={TF-SNO: Time-Frequency Gated Spectral Neural Operators for Learning Non-Stationary Partial Differential Equations},
  author={Zhou, Yitian and Zhang, Chaoning and Huang, Zhenzhen and Yu, Haoxuan and Zhang, Jiaquan and Li, Yiran and Mo, Fan and Liu, Kuien and Zou, Jie and Qin, Caiyan and others},
  journal={arXiv preprint arXiv:2606.21189},
  year={2026}
}

@article{zhang2026autoregression,
  title={Autoregression-Free Neural Operators for Time-Dependent PDEs},
  author={Zhang, Jiaquan and Qin, Caiyan and Bian, Haoyu and Cai, Libin and Lu, Yi and Zhang, Chaoning and Dong, Wei and Guo, Yuanfang and Yang, Yang and Shen, Heng Tao},
  journal={arXiv preprint arXiv:2605.25413},
  year={2026}
}

@article{karumuri2026physics,
  title={Physics-informed latent neural operator for real-time predictions of time-dependent parametric PDEs},
  author={Karumuri, Sharmila and Graham-Brady, Lori and Goswami, Somdatta},
  journal={Computer Methods in Applied Mechanics and Engineering},
  volume={450},
  pages={118599},
  year={2026},
  publisher={Elsevier}
}

@inproceedings{le2025learning,
  title={Learning a neural solver for parametric PDEs to enhance physics-informed methods},
  author={Le Boudec, Lise and De B{\'e}zenac, Emmanuel and Serrano, Louis and Regueiro-Espino, Ramon Daniel and Yin, Yuan and others},
  booktitle={International Conference on Learning Representations},
  volume={2025},
  pages={46994--47045},
  year={2025}
}

@article{takamoto2022pdebench,
  title={Pdebench: An extensive benchmark for scientific machine learning},
  author={Takamoto, Makoto and Praditia, Timothy and Leiteritz, Raphael and MacKinlay, Daniel and Alesiani, Francesco and Pfl{\"u}ger, Dirk and Niepert, Mathias},
  journal={Advances in neural information processing systems},
  volume={35},
  pages={1596--1611},
  year={2022}
}

@inproceedings{takamoto2023learning,
  title={Learning neural pde solvers with parameter-guided channel attention},
  author={Takamoto, Makoto and Alesiani, Francesco and Niepert, Mathias},
  booktitle={International Conference on Machine Learning},
  pages={33448--33467},
  year={2023},
  organization={PMLR}
}

\appendix

\section{Related Work}
\label{app:rl}
\subsection{Neural Operators} 
Physics-informed neural networks (PINNs)~\cite{raissi2019physics,guo2025advances,liu2023joint,luo2025physics} solve individual PDE instances by enforcing governing equations during training, but their instance-wise optimization is costly in many-query settings~\cite{gao2021phygeonet,wang2021understanding}. Their training can also become difficult for conservation laws and multiscale dynamics with discontinuities or sharp gradients~\cite{mao2020physics,jagtap2020conservative}. Neural operators address this limitation by learning function-to-function mappings that generalize across initial conditions, boundary conditions, and equation parameters~\cite{huang2026rethinking,zhang2026geometric}.
To address the generalization limitations of instance-specific solvers, the research focus shifts toward operator learning, aiming to map between infinite-dimensional function spaces. 
Lu et al.~\cite{lu2021learning,lu2019deeponet} propose DeepONet based on the universal operator approximation theorem, utilizing a dual branch-trunk architecture to learn continuous operators. 
Subsequently, Li et al.~\cite{DBLP:conf/iclr/LiKALBSA21} introduce the FNO, leveraging Fourier transforms for global frequency-domain convolutions to achieve efficient, resolution-independent solutions. 
Extensions such as Geo-FNO~\cite{DBLP:conf/iclr/TranMXO23}, U-NO~\cite{DBLP:journals/tmlr/RahmanRA23}, and linear--nonlinear separated convolutional solvers~\cite{qu2022learning} further enhance this framework for complex geometries, multiscale features, and time-dependent PDEs. However, most neural operators predominantly rely on globally shared parameters, implicitly assuming spatial homogeneity. 
This limitation hinders their ability to capture local dynamics in spatially heterogeneous PDEs, resulting in compromised accuracy or excessive model capacity.

\subsection{Adaptive and Hypernetwork-Based Neural Operators}

Physics-informed neural operators incorporate governing-equation constraints to reduce data dependence and improve physical consistency~\cite{karumuri2026physics}.
Competitive methods, including the physics-informed neural operator (PINO)~\cite{li2024physics} and physics-informed DeepONet (PI-DeepONet)~\cite{wang2021learning}, impose partial differential equation (PDE) residuals during operator training but retain globally shared backbones.
Attention-based mechanisms~\cite{hemmasian2024multi} and domain-decomposition approaches~\cite{zhou2025rapid} introduce spatial adaptivity through feature modulation or computational partitioning, whereas the underlying operator parameters remain largely shared.
HyperDeepONet~\cite{lee2023hyperdeeponet} and HyPINO~\cite{bischof2026hypino} instead generate target-network parameters from the input function or PDE specification.
However, both methods produce one parameter set per problem instance rather than parameters that vary continuously across spatial locations.
This gap motivates spatially adaptive operator parameterization at the local mapping level.

\section{Method}
\subsection{Code Interpolation}
\label{One-dimensional}
Take $\Omega=[a,b]$ and an element $\Omega_{i}=[x_{i},x_{i+1}]$ with
$K_{i}=2$, whose endpoint anchors are
$w_{i,1}=H_{\theta}(\mathrm{Emb}(x_{i}))$ and
$w_{i,2}=H_{\theta}(\mathrm{Emb}(x_{i+1}))$. With the normalized local
coordinate $s=(x-x_{i})/(x_{i+1}-x_{i})\in[0,1]$,
Eq.~\eqref{eq:param_interpolation} reduces to
$w_{i}(x)=\ell_{1}(s)w_{i,1}+\ell_{2}(s)w_{i,2}$ with the linear shape
functions $\ell_{1}(s)=1-s$ and $\ell_{2}(s)=s$. This construction mirrors
linear finite elements, except that a finite element interpolates nodal
function values whereas HNE interpolates the local conditioning code. Because
adjacent elements share nodal anchors,
$\lim_{x\to x_{i}^{-}}w(x)=\lim_{x\to x_{i}^{+}}w(x)=w_{i,1}$, so the
code field is at least $C^{0}$ across the domain and the local operator
does not jump at element interfaces.

\subsection{Theoretical Validation}
\label{sec:approx_theory}

Throughout this subsection, $\mathcal{G}_h(f)$ denotes the assembled
HNE--PoU prediction defined in Eq.~\eqref{eq:global-output}, which is
denoted by $\widehat{u}$ in the main text. Similarly,
$\mathcal{G}_{i,h}(f)$ denotes the local prediction $u_i$ in
Eq.~\eqref{eq:local_realization}.
The admissible input set $\mathcal{K}\subset X$ collects the PDE inputs
of interest.
The fixed partition $\{D_r\}_{r=1}^{R}$ introduced below characterizes
the regularity of the target operator $\mathcal{G}$. It is a property
of $\mathcal{G}$ and is distinct from both the overlapping cover
$\{\Omega_i\}$ introduced in Section~\ref{sec:hne:domain} and the $P_1$ mesh used for
code interpolation in Section~\ref{sec:hne:interpolation}.
Let $\Omega\subset\mathbb{R}^d$ be the spatial domain, let
$\mathcal{K}\subset X$ be the set of admissible PDE inputs, and let
$\mathcal{G}:\mathcal{K}\rightarrow Y$ be the target solution operator.
Assume that $\Omega$ admits a fixed finite partition
$\Omega=\bigcup_{r=1}^{R}D_r$, where the subdomains $D_r$ are closed
and have pairwise disjoint interiors.

For every $f\in\mathcal{K}$ and every $r$, the restriction
$\mathcal{G}(f)|_{D_r}$ satisfies the uniform Lipschitz estimate
\begin{equation}
\label{eq:piecewise-lipschitz}
    \left|
        \mathcal{G}(f)(x)-\mathcal{G}(f)(y)
    \right|
    \leq
    L_{\mathcal{G},r}\|x-y\|,
    \qquad
    x,y\in D_r,
\end{equation}
where $L_{\mathcal{G},r}$ is independent of $f$. We write
$L_{\mathcal{G}}
:=\max_{1\leq r\leq R}L_{\mathcal{G},r}$.

\subsubsection{Assumptions}

\begin{assumption}[Uniform Operator Approximation at Sampling Points]
\label{ass:anchor}
Consider a sequence of HNE meshes indexed by their maximal element
diameter $h$. For each $h$, assume that there exists a single shared
hypernetwork parameter vector $\theta_h$ such that every anchor
code is generated according to
\begin{equation}
\label{eq:shared-anchor-parameters}
    w_{i,k}^h
    =
    H_{\theta_h}
    \left(
        \operatorname{Emb}(x_{i,k}^h)
    \right).
\end{equation}

The corresponding operator error at the anchor points satisfies
\begin{equation}
\label{eq:uniform-anchor-error}
    \sup_{f\in\mathcal{K}}
    \max_{i,k}
    \left|
        G_{\nu}
        \left(
            f|_{\Omega_i},
            x_{i,k}^h,
            w_{i,k}^h
        \right)
        -
        \mathcal{G}(f)(x_{i,k}^h)
    \right|
    \leq
    \delta_h.
\end{equation}
\end{assumption}

\begin{assumption}[$P_1$ Code-Space Interpolation]
\label{ass:interpolation}
Each subregion $\Omega_i$ carries a simplicial $P_1$ mesh. Let $E$ be
an element of this mesh and let $x\in E$. The HNE code field is
obtained by nodal interpolation:
\begin{equation}
\label{eq:parameter-interpolation}
    w_{i,h}(x)
    =
    \sum_{x_{i,k}^h\in\mathcal{V}(E)}
    \alpha_{i,k}^h(x)\,w_{i,k}^h,
\end{equation}
where $\mathcal{V}(E)$ denotes the vertex set of $E$.

The interpolation weights satisfy
\begin{equation}
\label{eq:p1-properties}
\begin{aligned}
    \alpha_{i,k}^h(x)
    &\geq 0,
    \\[2mm]
    \sum_{x_{i,k}^h\in\mathcal{V}(E)}
    \alpha_{i,k}^h(x)
    &=1,
    \\[2mm]
    w_{i,h}(x_{i,k}^h)
    &=w_{i,k}^h.
\end{aligned}
\end{equation}

The meshes are aligned with the fixed partition
$\{D_r\}_{r=1}^{R}$, so that no mesh element crosses a partition
interface. For these $P_1$ weights, the interpolation-stability
condition holds automatically with $\Lambda_i=1$.
The full HNE construction may additionally assemble neighboring
elements conformingly to obtain a globally continuous code field.
Such cross-element conformity is not required for the estimate below.
\end{assumption}

\begin{assumption}[Non-Expansive Partition-of-Unity Aggregation]
\label{ass:pou}
Let $\{\Omega_i\}_{i=1}^{M}$ be an open cover of $\Omega$. A collection
of nonnegative functions $\{\phi_i\}_{i=1}^{M}$ forms a partition of
unity subordinate to this cover if
\begin{equation}
\label{eq:pou-support}
    \operatorname{supp}(\phi_i)
    \subset
    \Omega_i
\end{equation}
and
\begin{equation}
\label{eq:pou-unity}
    \sum_{i=1}^{M}\phi_i(x)
    =
    1,
    \qquad
    x\in\Omega.
\end{equation}
\end{assumption}

\begin{assumption}[Uniform Stability of Code Generation and the Conditional Local Map]
\label{ass:uniform-stability}
Define the code-generating map by
$q_{\theta_h}(x):=H_{\theta_h}(\operatorname{Emb}(x))$.
Assume that there exists a convex code set $W_0\subset W=\mathbb{R}^{p}$
containing $q_{\theta_h}(\Omega)$ for every $h$. Since the $P_1$
interpolants are convex combinations of nodal values, $w_{i,h}(x)$ also
belongs to $W_0$. The shared local MLP $G_{\nu}$ takes the
code $w$ as a conditioning input, and for each fixed local
input $f|_{\Omega_i}$ and query location $x$, the conditional map
$w\mapsto G_{\nu}(f|_{\Omega_i},x,w)$ is continuous on the
compact code space $W_0$.

Assume that there exist constants $L_q,L_x,L_w<\infty$, independent of
$h$, such that
\begin{equation}
\label{eq:q-lipschitz}
    \left\|
        q_{\theta_h}(x)-q_{\theta_h}(y)
    \right\|_W
    \leq
    L_q\|x-y\|.
\end{equation}

Moreover, for every $f\in\mathcal{K}$, every $x,y\in\Omega_i$, and
every $w,v\in W_0$, assume that
\begin{equation}
\label{eq:realization-lipschitz}
\begin{aligned}
    &
    \left|
        G_{\nu}
        \left(
            f|_{\Omega_i},
            x,
            w
        \right)
        -
        G_{\nu}
        \left(
            f|_{\Omega_i},
            y,
            v
        \right)
    \right|
    \\
    &\qquad
    \leq
    L_x\|x-y\|
    +
    L_w\|w-v\|_W.
\end{aligned}
\end{equation}
\end{assumption}

\subsubsection{Elementwise Code Control}

\begin{proposition}[Elementwise Control of the HNE Code Field]
\label{prop:elementwise}
Let $E$ be a $P_1$ element with diameter $h_E$, let $x\in E$, and let
$x_{i,k}^h$ be any vertex of $E$. Then
\begin{equation}
\label{eq:elementwise-parameter-bound}
    \left\|
        w_{i,h}(x)
        -
        w_{i,h}(x_{i,k}^h)
    \right\|_W
    \leq
    L_q h_E.
\end{equation}
\end{proposition}

\begin{proof}
By the nodal interpolation property,
$w_{i,h}(x_{i,k}^h)
=w_{i,k}^h
=q_{\theta_h}(x_{i,k}^h)$.
Using the convex-combination representation of $w_{i,h}(x)$ gives
\begin{equation}
\begin{aligned}
    &
    \left\|
        w_{i,h}(x)
        -
        w_{i,h}(x_{i,k}^h)
    \right\|_W
    \\
    &=
    \left\|
        \sum_{x_{i,\ell}^h\in\mathcal{V}(E)}
        \alpha_{i,\ell}^h(x)
        \left(
            q_{\theta_h}(x_{i,\ell}^h)
            -
            q_{\theta_h}(x_{i,k}^h)
        \right)
    \right\|_W
    \\
    &\leq
    \sum_{x_{i,\ell}^h\in\mathcal{V}(E)}
    \alpha_{i,\ell}^h(x)
    \left\|
        q_{\theta_h}(x_{i,\ell}^h)
        -
        q_{\theta_h}(x_{i,k}^h)
    \right\|_W
    \\
    &\leq
    L_q
    \sum_{x_{i,\ell}^h\in\mathcal{V}(E)}
    \alpha_{i,\ell}^h(x)
    \left\|
        x_{i,\ell}^h-x_{i,k}^h
    \right\|
    \\
    &\leq
    L_q h_E.
\end{aligned}
\end{equation}
\end{proof}

\subsubsection{Local and Global Predictions}

The local prediction associated with the $i$th subregion is defined by
\begin{equation}
\label{eq:local-output}
    \mathcal{G}_{i,h}(f)(x)
    :=
    G_{\nu}
    \left(
        f|_{\Omega_i},
        x,
        w_{i,h}(x)
    \right).
\end{equation}

The partition-of-unity assembly is
\begin{equation}
\label{eq:global-output}
    \mathcal{G}_h(f)(x)
    :=
    \sum_{i=1}^{M}
    \phi_i(x)\,
    \mathcal{G}_{i,h}(f)(x).
\end{equation}

The global mesh size is defined as
\begin{equation}
\label{eq:mesh-size}
    h
    :=
    \max_i
    \max_{E\subset\Omega_i}
    \operatorname{diam}(E).
\end{equation}

\begin{proposition}[Pointwise Non-Amplification and Locality of PoU]
\label{prop:pou}
Define the $i$th local error by
$e_i(f,x)
:=\mathcal{G}_{i,h}(f)(x)-\mathcal{G}(f)(x)$,
and define the active index set by
$\mathcal{I}(x):=\{i:\phi_i(x)>0\}$. Then
\begin{equation}
\label{eq:pou-worst-case}
    \left|
        \mathcal{G}_h(f)(x)-\mathcal{G}(f)(x)
    \right|
    \leq
    \max_{i\in\mathcal{I}(x)}
    |e_i(f,x)|.
\end{equation}

Moreover, the direct contribution of the $i$th local error vanishes
outside $\operatorname{supp}(\phi_i)$.
\end{proposition}

\begin{proof}
Because the weights are nonnegative and sum to one,
\begin{equation}
\begin{aligned}
    \left|
        \mathcal{G}_h(f)(x)-\mathcal{G}(f)(x)
    \right|
    &=
    \left|
        \sum_{i\in\mathcal{I}(x)}
        \phi_i(x)e_i(f,x)
    \right|
    \\
    &\leq
    \sum_{i\in\mathcal{I}(x)}
    \phi_i(x)|e_i(f,x)|
    \\
    &\leq
    \max_{i\in\mathcal{I}(x)}
    |e_i(f,x)|.
\end{aligned}
\end{equation}

The locality statement follows immediately from the fact that
$\phi_i(x)=0$ whenever
$x\notin\operatorname{supp}(\phi_i)$.
\end{proof}

\subsubsection{Theorem and Proof}

\begin{theorem}[Consistency of Hyper-Neural Element Assembly]
\label{thm:consistency}
Under Assumptions~\ref{ass:anchor}--\ref{ass:uniform-stability}, the
assembled HNE--PoU model satisfies
\begin{equation}
\label{eq:main-bound}
\begin{aligned}
    \sup_{f\in\mathcal{K}}
    \left\|
        \mathcal{G}_h(f)-\mathcal{G}(f)
    \right\|_{L^\infty(\Omega)}
    \leq
    \delta_h
    +
    \left(
        L_x+L_wL_q+L_{\mathcal{G}}
    \right)h.
\end{aligned}
\end{equation}

Consequently, if $\delta_h\rightarrow0$ and $h\rightarrow0$, then
\begin{equation}
\label{eq:consistency-limit}
    \sup_{f\in\mathcal{K}}
    \left\|
        \mathcal{G}_h(f)-\mathcal{G}(f)
    \right\|_{L^\infty(\Omega)}
    \longrightarrow
    0.
\end{equation}
\end{theorem}

\begin{proof}
Fix an input $f\in\mathcal{K}$, a local subregion $\Omega_i$, and a
point $x$ contained in an HNE element $E\subset D_r$. Choose any vertex
$x_{i,k}^h$ of $E$.

Adding and subtracting the anchor prediction and the target value at
the anchor gives
\begin{equation}
\label{eq:three-term-error}
\begin{aligned}
    &
    \left|
        \mathcal{G}_{i,h}(f)(x)
        -
        \mathcal{G}(f)(x)
    \right|
    \\
    &\leq
    \left|
        G_{\nu}
        \left(
            f|_{\Omega_i},
            x,
            w_{i,h}(x)
        \right)
        -
        G_{\nu}
        \left(
            f|_{\Omega_i},
            x_{i,k}^h,
            w_{i,k}^h
        \right)
    \right|
    \\
    &\quad+
    \left|
        G_{\nu}
        \left(
            f|_{\Omega_i},
            x_{i,k}^h,
            w_{i,k}^h
        \right)
        -
        \mathcal{G}(f)(x_{i,k}^h)
    \right|
    \\
    &\quad+
    \left|
        \mathcal{G}(f)(x_{i,k}^h)
        -
        \mathcal{G}(f)(x)
    \right|.
\end{aligned}
\end{equation}

By Assumption~\ref{ass:uniform-stability},
Proposition~\ref{prop:elementwise}, and the inequality $h_E\leq h$,
the first term on the right-hand side of
Eq.~\eqref{eq:three-term-error} is bounded by
$(L_x+L_wL_q)h$.

Assumption~\ref{ass:anchor} bounds the second term by $\delta_h$.
Because the mesh is aligned with the fixed partition, both $x$ and
$x_{i,k}^h$ belong to the same subdomain $D_r$. Therefore,
Eq.~\eqref{eq:piecewise-lipschitz} bounds the third term by
$L_{\mathcal{G},r}h$. Combining these estimates yields
\begin{equation}
\label{eq:local-bound}
    \left|
        \mathcal{G}_{i,h}(f)(x)
        -
        \mathcal{G}(f)(x)
    \right|
    \leq
    \delta_h
    +
    \left(
        L_x+L_wL_q+L_{\mathcal{G},r}
    \right)h.
\end{equation}

Since $L_{\mathcal{G},r}\leq L_{\mathcal{G}}$,
Proposition~\ref{prop:pou} gives the same upper bound for the assembled
prediction:
\begin{equation}
    \left|
        \mathcal{G}_h(f)(x)-\mathcal{G}(f)(x)
    \right|
    \leq
    \delta_h
    +
    \left(
        L_x+L_wL_q+L_{\mathcal{G}}
    \right)h.
\end{equation}
Taking the supremum over $x\in\Omega$ and $f\in\mathcal{K}$ proves
Eq.~\eqref{eq:main-bound}. The convergence statement follows when both
$\delta_h$ and $h$ tend to zero.
\end{proof}

Theorem~\ref{thm:consistency} is a conditional, one-step consistency
result for the model sequence $\{\mathcal{G}_h\}_h$. It assumes, rather
than proves, that at each resolution a shared hypernetwork achieves the
anchor error $\delta_h$, while the constants in
Assumption~\ref{ass:uniform-stability} remain uniform with respect to
$h$.
The result does not establish optimization convergence, long-horizon
rollout stability, or zero-shot cross-resolution transfer for a single
fixed trained model. The fixed partition is intended to describe
persistent spatial heterogeneity. Problems with moving interfaces
require additional state-dependent structure.
For a target containing a genuine jump discontinuity, a continuous
HNE--PoU assembly cannot generally converge in the global
$L^\infty(\Omega)$ norm. Proposition~\ref{prop:pou} nevertheless
remains valid: a single assembly step does not exceed the largest active
local error, and a local error does not directly propagate outside the
support of its PoU weight. Problems with jump discontinuities should
therefore be analyzed either away from the jump set or in an $L^p$
norm.

\section{Experiments}

\subsection{Implementation Details}
\label{app:implementation_details}

All experiments are implemented in PyTorch and conducted on a single NVIDIA RTX 4090 GPU (24\,GB).
For each benchmark, we use $N_{\mathrm{s}}=1{,}000$ training trajectories and $200$ independently generated test trajectories, where the training, validation, and test splits are produced with different random seeds so that near-duplicate initial conditions do not appear across splits.
All models are trained for $100$ epochs with the AdamW optimizer \cite{loshchilov2017decoupled}, where a cosine annealing scheduler \cite{loshchilov2017sgdr} decays the learning rate during training.
All models are trained under the same pushforward strategy with an unrolling horizon of $T=5$, so that the comparison isolates the operator architecture rather than the training protocol.
All baselines except HyPINO use the MSE objective, and HyPINO additionally uses its PDE residual.
To enforce resolution invariance, the inputs are downsampled during training to $N_x=256$ for the 1D benchmarks and to $N_x=N_y=64$ for the 2D benchmarks, while the 3D benchmark is used at its native resolution.
At inference, the input is downsampled to the training resolution and the output is upsampled back to the original resolution, which preserves fine-scale details and exploits the zero-shot super-resolution property of the spectral backbone \cite{DBLP:conf/iclr/LiKALBSA21}.

The SANO-specific hyperparameters are summarized in Table~\ref{tab:hyperparameters} and are grouped by spatial dimension, since the number of subregions, the sampling budget per subregion, and the width of the shared local MLP must scale with the dimension of the domain.
The domain is covered by $M$ overlapping subregions with a $25\%$ overlap ratio in every dimension, which guarantees that the partition-of-unity weights $\phi_i$ remain smooth across subregion boundaries.
Within each subregion, the code field is reconstructed from $K_i$ sampling points by linear interpolation, and the local predictions are assembled using cosine partition-of-unity weights.
The Fourier feature matrix $\B{B}\in\mathbb{R}^{m\times d}$ has entries drawn from $\mathcal{N}(0,\sigma_B^2)$, where $m$ controls the embedding capacity and $\sigma_B$ controls the highest spatial frequency that the hypernetwork can resolve.
The hypernetwork $H_\theta$ is a fully connected network whose output dimension $p$ specifies the dimension of the spatial code. This code is concatenated with the local input representation to condition the shared local MLP, whose width and depth are reported separately.
We use a smaller local-MLP width, a smaller batch size, and a lower learning rate on the 3D benchmark to accommodate the memory footprint of the multi-channel volumetric fields.
The learnable parameters are initialized as follows: the hypernetwork output layer to zero, so that the initial code field is spatially uniform, and all remaining weights with Kaiming uniform initialization.

\begin{table}[t]
\centering
\caption{Complete SANO hyperparameter configuration by spatial dimension.}
\label{tab:hyperparameters}
\footnotesize
\setlength{\tabcolsep}{8.0pt}
\renewcommand{\arraystretch}{1}
\begin{tabular}{llll}
\toprule
Hyperparameter & 1D & 2D & 3D \\
\midrule
Number of subregions $M$        & $16$   & $8\times8$ & $4\times4\times4$ \\
Subregion overlap ratio         & $25\%$ & $25\%$     & $25\%$ \\
Sampling points $K_i$           & $4$    & $9$        & $8$ \\
Code interpolation              & Linear & Linear     & Linear \\
PoU weight $\phi_i$             & Cosine & Cosine     & Cosine \\
\midrule
Fourier features $m$            & $64$   & $128$      & $96$ \\
Frequency scale $\sigma_B$      & $4$    & $4$        & $3$ \\
\midrule
Hypernetwork depth              & $3$    & $4$        & $4$ \\
Hypernetwork width              & $128$  & $256$      & $256$ \\
Code dimension $p$              & $128$  & $256$      & $256$ \\
\midrule
Local MLP width                 & $64$   & $64$       & $32$ \\
Local MLP depth                 & $4$    & $4$        & $4$ \\
\midrule
Batch size                      & $16$   & $8$        & $2$ \\
Learning rate                   & $10^{-3}$ & $10^{-3}$ & $5\times10^{-4}$ \\
$\lambda_{\mathrm{reg}}$        & $10^{-4}$ & $10^{-4}$ & $10^{-4}$ \\
\bottomrule
\end{tabular}
\end{table}
 
\subsection{Detailed Benchmark PDEs}
\label{app:benchmark_pdes}

This section provides the governing equations, heterogeneous-field
configurations, and physical motivation for the five benchmarks introduced
in Section~\ref{sec:benchmark_pdes}. All benchmarks use periodic boundary
conditions. Unless stated otherwise, the heterogeneous coefficients follow
the common construction in Eq.~\eqref{eq:heterogeneity}. Each heterogeneous
field is generated once and remains fixed across all trajectories and data
splits.

\paragraph{Burgers equation.}
The Burgers equation \cite{BLAW:18:PRAMANA} evaluates nonlinear transport
under spatially varying dissipation. It is written in conservative form as
\begin{equation}
\partial_t u + u\partial_x u
=
\partial_x\left(\nu(x)\partial_x u\right).
\end{equation}
The viscosity $\nu(x)$ follows Eq.~\eqref{eq:heterogeneity} with
$\nu_0=10^{-2}$ and $\alpha_\nu=0.5$, and is clipped to
\[
\nu(x)\in
\left[5\times10^{-3},\,1.5\times10^{-2}\right].
\]
Its modulation field contains three Gaussian bumps centered at
$\lrA{\pi/2,\pi,3\pi/2}$, with amplitudes
$\lrA{1.0,-0.7,0.9}$ and periodic widths
$\lrA{0.35,0.50,0.30}$.
Low-viscosity regions develop steeper, shock-like structures, whereas
high-viscosity regions are smoothed more rapidly. The benchmark therefore
tests whether an operator can model region-dependent gradient sharpening.

\paragraph{Kuramoto--Sivashinsky equation.}
The Kuramoto--Sivashinsky equation \cite{KT:76:PTP} evaluates chaotic
spatiotemporal evolution under spatially nonuniform excitation. We use the
forced formulation
\begin{equation}
\label{eq:ks}
\partial_t u
+
u\partial_x u
+
\partial_{xx}u
+
\partial_{xxxx}u
=
f_{\mathrm{KS}}(x).
\end{equation}
Here, $\partial_{xx}u$ induces long-wave instability and
$\partial_{xxxx}u$ provides small-scale dissipation.
The fixed forcing $f_{\mathrm{KS}}(x)$ is composed of three sinusoidal
modes on a domain of length $L=64$, with amplitude $A_f=0.05$.
It has zero spatial mean and therefore introduces persistent
region-dependent excitation without changing the global mean state.
This setting tests long-horizon spectral prediction under a fixed spatial
bias.

\paragraph{Allen--Cahn equation.}
The Allen--Cahn equation \cite{allen1979microscopic} is used to study
interface-dominated evolution in a heterogeneous bistable medium:
\begin{equation}
\partial_t u
=
\varepsilon_{\mathrm{ac}}\Delta u
-
\gamma(\B{x})\left(u^3-u\right).
\end{equation}
The interfacial diffusion coefficient is
$\varepsilon_{\mathrm{ac}}=10^{-3}$.
The reaction rate follows Eq.~\eqref{eq:heterogeneity} with
$\gamma_0=1.0$ and $\alpha_\gamma=0.5$, and is clipped to
\[
\gamma(\B{x})\in[0.5,1.5].
\]
The modulation field combines a
$\sin(2\pi x)\sin(2\pi y)$ component with two localized Gaussian bumps of
opposite signs.
Spatial variation in $\gamma(\B{x})$ produces region-dependent interface
contraction and migration, making this benchmark sensitive to the
preservation of sharp moving interfaces.

\paragraph{Fisher--KPP equation.}
The Fisher--KPP equation
\cite{fisher1937wave,kolmogorov1937study} evaluates nonlinear
reaction--diffusion front propagation through a heterogeneous environment:
\begin{equation}
\partial_t u
=
D\Delta u
+
r(\B{x})u(1-u).
\end{equation}
We set $D=10^{-3}$. The growth rate $r(\B{x})$ follows
Eq.~\eqref{eq:heterogeneity} with $r_0=1.0$ and $\alpha_r=0.5$, and is
clipped to
\[
r(\B{x})\in[0.5,1.5].
\]
No time-dependent source term is introduced, so $r(\B{x})$ is the only
persistent spatial structure.
Reaction fronts propagate faster through high-growth regions and slow down
in low-growth regions. The benchmark therefore evaluates front-position
accuracy under a spatially varying propagation speed.

\paragraph{Three-dimensional compressible flow.}
The 3D compressible-flow benchmark evaluates scalability to strongly
coupled, multi-field conservation-law dynamics. Its governing equations are
\begin{subequations}
\label{eq:cf}
\begin{align}
\partial_t\rho
+
\M{\nabla}\cdot(\rho\B{v})
&=0,\\
\partial_t(\rho\B{v})
+
\M{\nabla}\cdot
\left(
\rho\B{v}\B{v}^{\top}
+
p\B{I}_3
\right)
&=0,\\
\partial_t\mathcal{E}
+
\M{\nabla}\cdot
\left(
(\mathcal{E}+p)\B{v}
\right)
&=0.
\end{align}
\end{subequations}
Here, $\rho$ denotes density, $\B{v}\in\mathbb{R}^3$ velocity,
$p$ pressure, $\mathcal{E}$ total energy, and $\B{I}_3$ the
$3\times3$ identity matrix. The system is closed by
\begin{equation}
p
=
(\gamma_{\mathrm{cf}}-1)
\left(
\mathcal{E}
-
\frac{1}{2}\rho\|\B{v}\|_2^2
\right),
\qquad
\gamma_{\mathrm{cf}}=\frac{5}{3}.
\end{equation}

Spatial heterogeneity is introduced through a fixed perturbation-gain field
$m_{\mathrm{cf}}(\B{x})\in[0.6,1.4]$. It follows
Eq.~\eqref{eq:heterogeneity} with unit nominal value and
$\alpha_m=0.4$, and modulates the amplitude of the random initial
perturbations. Different regions consequently exhibit persistent
differences in their initial energy statistics without requiring changes to
the computational geometry.
This benchmark evaluates spatial adaptivity in a three-dimensional system
where density, velocity, pressure, and total energy evolve jointly.

\begin{table}[h]
\centering
\caption{Spatially heterogeneous PDE benchmarks.}
\label{tab:pde_benchmarks}
\small
\renewcommand{\arraystretch}{1}
\begin{threeparttable}
\begin{tabular}{llllll}
\toprule
PDE & Dim. & Domain & Resolution & Time; frames & Heterogeneity \\
\midrule
Burgers & 1D & $[0,2\pi)$ & $1024$ & $[0,1]$; $101$
        & $\nu(x)$ \\
KS & 1D & $[0,64)$ & $1024$ & $[0,10]$; $101$
   & $f_{\mathrm{KS}}(x)$ \\
AC & 2D & $[0,1)^2$ & $64\times64$ & $[0,1]$; $21$
   & $\gamma(\B{x})$ \\
FKPP & 2D & $[0,1)^2$ & $64\times64$ & $[0,1]$; $21$
     & $r(\B{x})$ \\
CF & 3D & $[0,1)^3$ & $32^3$ & $[0,1]$; $21$
   & $m_{\mathrm{cf}}(\B{x})$ \\
\bottomrule
\end{tabular}
\end{threeparttable}
\end{table}

\subsection{Dataset Generation Details}
\label{app:dataset_generation}

For every benchmark, the spatially heterogeneous coefficient or forcing
field is generated once and held fixed across the training, validation,
and test sets. Individual trajectories differ only through their randomly
sampled initial conditions. This construction isolates the effect of the
persistent heterogeneous medium while preventing the model from directly
observing the underlying coefficient field.

\paragraph{One-dimensional initial conditions.}
For the Burgers and Kuramoto--Sivashinsky benchmarks, the initial state is
sampled as a random multi-frequency Fourier profile,
\begin{equation}
\label{eq:initial_condition_1d}
u_0(x)
=
A_0
\sum_{k=1}^{K}
a_k k^{-\beta}
\sin\left(kx+\phi_k\right).
\end{equation}
The Fourier coefficients and phases are sampled according to
\begin{equation}
a_k\sim\mathcal{N}(0,1),
\qquad
\phi_k\sim\mathcal{U}(0,2\pi).
\end{equation}
The spectral-decay parameter $\beta$ controls the relative contribution of
high-frequency components.

For Burgers, we use
\begin{equation}
K=8,
\qquad
\beta\sim\mathcal{U}(1.2,2.0),
\qquad
A_0\sim\mathcal{U}(0.8,1.5).
\end{equation}
The sampled profile is mean-removed and amplitude-normalized.

For KS, we use
\begin{equation}
K=12,
\qquad
\beta\sim\mathcal{U}(1.0,1.8),
\qquad
A_0\sim\mathcal{U}(0.5,1.0).
\end{equation}
Only the spatial mean is removed. Compared with initial conditions formed
from a small set of fixed low-frequency modes, these distributions generate
transport and chaotic structures with more diverse spatial scales,
locations, and steepness.

\paragraph{Allen--Cahn initial conditions.}
For the Allen--Cahn benchmark, an auxiliary Gaussian random field is drawn
according to
\begin{equation}
g
\sim
\mathcal{N}
\left(
0,
\left(-\Delta+\tau^2\B{I}\right)^{-\alpha}
\right),
\end{equation}
where
\begin{equation}
\alpha\sim\mathcal{U}(2.0,3.0),
\qquad
\tau\sim\mathcal{U}(2.5,4.0).
\end{equation}
After standardization, the field is mapped to a bistable configuration:
\begin{equation}
u_0(\B{x})
=
\tanh\left(s\widetilde{g}(\B{x})+b\right)
+
\eta(\B{x}),
\end{equation}
where
\begin{equation}
s\sim\mathcal{U}(1.2,2.0),
\qquad
b\sim\mathcal{U}(-0.2,0.2),
\qquad
\eta(\B{x})\sim\mathcal{N}(0,0.03^2).
\end{equation}
The resulting field is clipped to $[-1,1]$. This procedure generates phase
configurations with varying numbers of interfaces, spatial locations, and
curvatures.

\paragraph{Fisher--KPP initial conditions.}
For the Fisher--KPP benchmark, each initial condition is constructed from
a random collection of localized Gaussian seeds:
\begin{equation}
u_0(\B{x})
=
\sum_{j=1}^{J}
A_j
\exp\left(
-\frac{
d_{\mathrm{per}}(\B{x},\B{\mu}_j)^2
}{
2\sigma_j^2
}
\right)
+
\eta(\B{x}),
\end{equation}
where $d_{\mathrm{per}}$ denotes the periodic distance and the centers
$\B{\mu}_j$ are sampled uniformly over the domain. The number, amplitudes,
and widths of the seeds are sampled as
\begin{equation}
J\sim\mathcal{U}\{2,3,4,5\},
\qquad
A_j\sim\mathcal{U}(0.4,1.0),
\qquad
\sigma_j\sim\mathcal{U}(0.04,0.12).
\end{equation}
The additive perturbation satisfies
\begin{equation}
\eta(\B{x})\sim\mathcal{N}(0,0.01^2).
\end{equation}
The final field is clipped to $[0,1]$. An additional randomly oriented
elliptical seed is included with a fixed sampling probability to increase
the diversity of initial front geometries.

\paragraph{Compressible-flow initial conditions.}
For the 3D compressible-flow benchmark, five independent smooth Gaussian
random fields are sampled in Fourier space. Their spectra satisfy
\begin{equation}
\widehat{g}(\B{k})
\propto
\xi_{\B{k}}
\left(
\|\B{k}\|_2^2+\tau^2
\right)^{-\alpha/2},
\end{equation}
where
\begin{equation}
\alpha\sim\mathcal{U}(2.5,4.0),
\qquad
\tau\sim\mathcal{U}(2.0,4.0),
\end{equation}
and $\xi_{\B{k}}$ is complex standard Gaussian noise subject to the
Hermitian symmetry required to obtain real-valued spatial fields.

Density and pressure are initialized as
\begin{equation}
\begin{aligned}
\rho(\B{x},0)
&=
\rho_0
\left[
1+
A_\rho
m_{\mathrm{cf}}(\B{x})
\widetilde{g}_{\rho}(\B{x})
\right],\\
p(\B{x},0)
&=
p_0
\left[
1+
A_p
m_{\mathrm{cf}}(\B{x})
\widetilde{g}_{p}(\B{x})
\right].
\end{aligned}
\end{equation}
The reference values and perturbation amplitudes are
\begin{equation}
\rho_0=1,
\qquad
p_0=\frac{1}{\gamma_{\mathrm{cf}}},
\qquad
A_\rho,A_p\sim\mathcal{U}(0.03,0.08).
\end{equation}
Each velocity component is generated from an independent smooth random
field, modulated by $m_{\mathrm{cf}}(\B{x})$, and then uniformly rescaled so
that the initial root-mean-square Mach number satisfies
\begin{equation}
\mathrm{Ma}_{\mathrm{rms}}=0.3.
\end{equation}
The total energy is recovered from the equation of state,
\begin{equation}
\mathcal{E}
=
\frac{p}{\gamma_{\mathrm{cf}}-1}
+
\frac{1}{2}\rho\|\B{v}\|_2^2,
\end{equation}
rather than sampled independently. For numerical stability, the primitive
variables are constrained by
\begin{equation}
\rho\geq0.2,
\qquad
p\geq10^{-4}.
\end{equation}

\paragraph{Numerical trajectory generation.}
The 1D Burgers and KS trajectories are generated using high-accuracy
spectral integration methods. For the 2D Allen--Cahn and Fisher--KPP
benchmarks, the simulator uses the internal step size
\begin{equation}
\Delta t_{\mathrm{sim}}=10^{-3},
\end{equation}
with $50$ numerical substeps between consecutive recorded frames.

The 3D compressible-flow trajectories are generated using a finite-volume
scheme with the Rusanov numerical flux, second-order
total-variation-diminishing Runge--Kutta time integration, and a CFL number
of $0.3$.

Only the recorded states listed in Table~\ref{tab:pde_benchmarks} are made
available to the learning algorithms. Consequently, the model learns an
initial-state-to-trajectory mapping under a spatially heterogeneous but
unobserved medium. It is not supplied with the coefficient or forcing field
and is therefore not trained as a coefficient-conditioned regression model.

\subsection{Long-horizon Rollout Performance}
\label{long}
To evaluate temporal stability beyond single-step prediction, we conduct autoregressive rollout experiments on the Burgers, KS, AC, FKPP, and CF benchmarks. 
Starting from the same initial condition, each model predicts the next solution state and feeds its own prediction back as the input for subsequent steps. 
We report MSE, Rel-\(H^1\), and Rel-\(L_2\) at multiple rollout checkpoints in Figure~\ref{fig:rollout_metrics}. 
The horizontal axis denotes the normalized prediction checkpoint along the rollout horizon, with larger values corresponding to later autoregressive steps, where error accumulation is typically more severe. 
Across all benchmarks and metrics, SANO consistently yields the lowest rollout errors, and the gap becomes more visible at later checkpoints where autoregressive error accumulation is more severe.
More specifically, SANO maintains a clear advantage on the 1D Burgers benchmark and remains substantially below the competing baselines on the chaotic KS dynamics, where several methods exhibit rapid growth in MSE and relative errors.  
The same trend extends to the 2D and 3D PDE case: despite the multi-channel fluid dynamics, SANO stays below FNO, UNO, CNO, HyperDeepONet, and HyPINO throughout the rollout horizon. 
These observations indicate that the spatially continuous code field provides a stable inductive bias for long-horizon PDE prediction.

\begin{figure*}[t]
    \centering
    \includegraphics[width=0.85\linewidth]{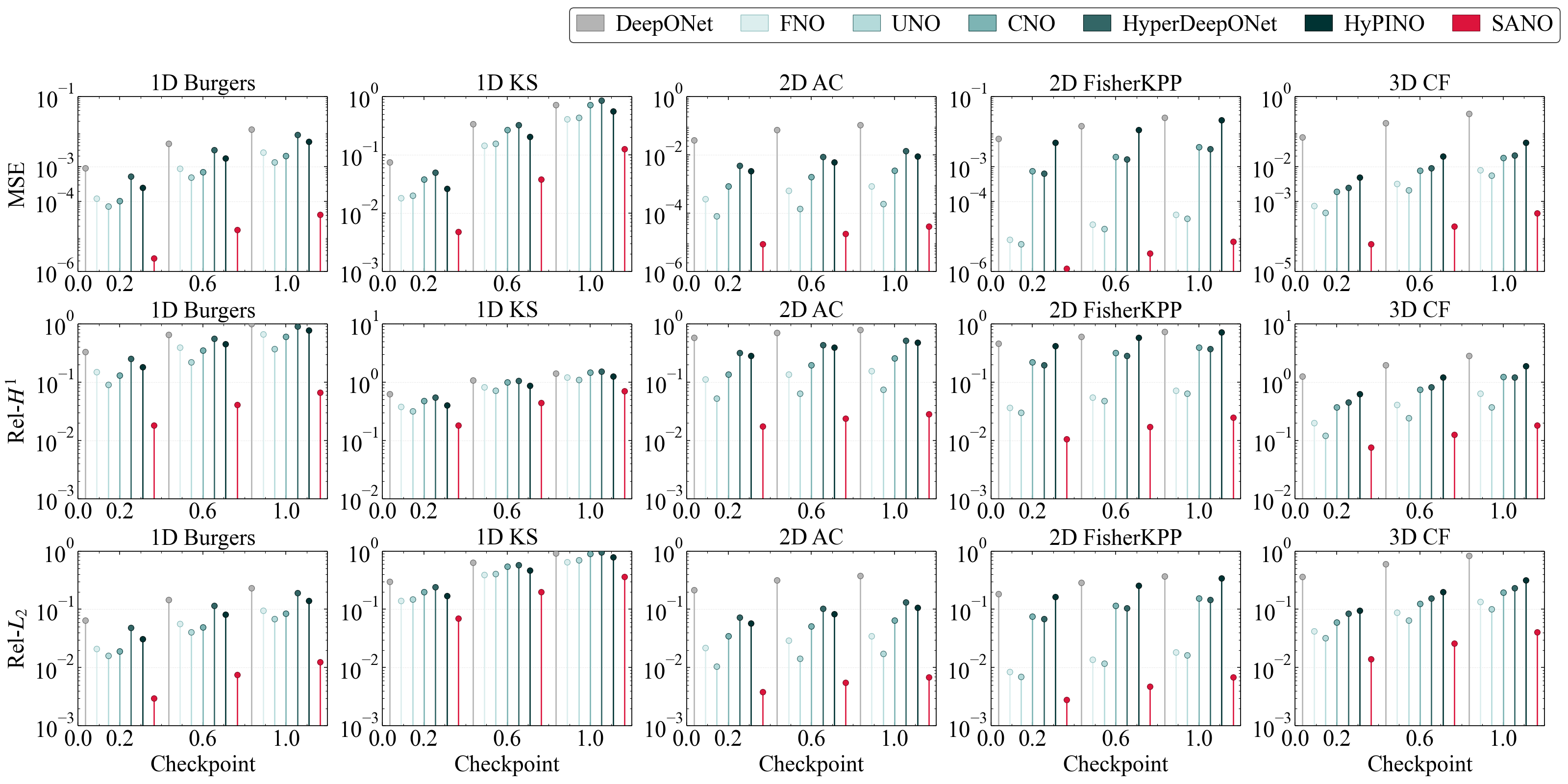}
    \caption{Long-horizon rollout errors across five PDE benchmarks. 
Rows report MSE, Rel-\(H^1\), and Rel-\(L_2\); columns denote different benchmarks. 
SANO maintains lower errors over prediction checkpoints.}
    \label{fig:rollout_metrics}
\end{figure*}

\subsection{Cross-parameter Generalization}
\label{subsec:cross_parameter_generalization}

To examine whether SANO can interpolate across unseen physical parameters, we conduct a cross-parameter experiment on the 2D Allen--Cahn benchmark at a spatial resolution of \(64\times64\). 
The training set mixes two diffusion coefficients, \(\varepsilon=5\times10^{-4}\) and \(\varepsilon=2\times10^{-3}\), with 500 samples for each coefficient, while all models are tested on 100 samples generated by the intermediate coefficient \(\varepsilon=10^{-3}\). 
This setting evaluates whether the learned operator can transfer across diffusion strengths rather than memorizing a single parameter regime.

\begin{table*}[t!]
\centering
\caption{Cross-parameter generalization results on the 2D AC benchmark}
\label{tab:cross_parameter_ac}
\begin{threeparttable}
\begin{tabular}{llllll}
\toprule
Model 
& F-RMSE\(\downarrow\) 
& RMSE\(\downarrow\) 
& MSE\(\downarrow\) 
& Rel-\(L_2\)\(\downarrow\) 
& Rel-\(H^1\)\(\downarrow\)\\
\midrule
\midrule
DeepONet      & $5.09 \times 10^{-1}$ & $3.30 \times 10^{-1}$ & $1.09 \times 10^{-1}$ & $4.04 \times 10^{-1}$ & $6.83 \times 10^{-1}$  \\
FNO           & $5.74 \times 10^{-2}$ & $4.35 \times 10^{-2}$ & $1.89 \times 10^{-3}$ & $4.92 \times 10^{-2}$ & $2.52 \times 10^{-1}$  \\
UNO           & $5.17 \times 10^{-2}$ & $3.78 \times 10^{-2}$ & $1.43 \times 10^{-3}$ & $3.55 \times 10^{-2}$ & $1.18 \times 10^{-1}$  \\
CNO           & $1.17 \times 10^{-1}$ & $8.55 \times 10^{-2}$ & $7.31 \times 10^{-3}$ & $8.33 \times 10^{-2}$ & $2.98 \times 10^{-1}$ \\
HyperDeepONet & $1.48 \times 10^{-1}$ & $1.06 \times 10^{-1}$ & $1.12 \times 10^{-2}$ & $1.14 \times 10^{-1}$ & $5.51 \times 10^{-1}$  \\
HyPINO        & $1.21 \times 10^{-1}$ & $8.80 \times 10^{-2}$ & $7.74 \times 10^{-3}$ & $9.82 \times 10^{-2}$ & $5.28 \times 10^{-1}$  \\
SANO          & \cellcolor[HTML]{C6ECE5}$2.67 \times 10^{-2}$ & \cellcolor[HTML]{C6ECE5}$1.86 \times 10^{-2}$ & \cellcolor[HTML]{C6ECE5}$3.46 \times 10^{-4}$ & \cellcolor[HTML]{C6ECE5}$2.04 \times 10^{-2}$ & \cellcolor[HTML]{C6ECE5}$6.04 \times 10^{-2}$ \\
\bottomrule
\end{tabular}
\end{threeparttable}
\end{table*}

\begin{table*}[t]
\centering
\caption{Cross-resolution generalization results on the 2D AC benchmark}
\setlength{\tabcolsep}{4.5pt}
\label{tab:cross_resolution_ac}
\begin{threeparttable}
\begin{tabular}{llllllll}
\toprule
Resolution & Model 
& MSE@1\(\downarrow\) 
& Rel-\(L_2\)@1\(\downarrow\) 
& Rel-\(H^1\)@1\(\downarrow\) 
& MSE@16\(\downarrow\) 
& Rel-\(L_2\)@16\(\downarrow\) 
& Rel-\(H^1\)@16\(\downarrow\) \\
\midrule
\midrule

\multirow{7}{*}{\(32\times32\)}
& DeepONet      & $1.47 \times 10^{-1}$ & $4.96 \times 10^{-1}$ & $8.07 \times 10^{-1}$ & $1.99 \times 10^{-1}$ & $5.43 \times 10^{-1}$ & $9.10 \times 10^{-1}$ \\
& FNO           & $3.46 \times 10^{-5}$ & $7.71 \times 10^{-3}$ & $2.80 \times 10^{-2}$ & $9.98 \times 10^{-4}$ & $3.80 \times 10^{-2}$ & $9.79 \times 10^{-2}$ \\
& UNO           & $1.93 \times 10^{-5}$ & $5.77 \times 10^{-3}$ & $2.09 \times 10^{-2}$ & $4.33 \times 10^{-4}$ & $2.53 \times 10^{-2}$ & $6.88 \times 10^{-2}$ \\
& CNO           & $1.53 \times 10^{-3}$ & $5.11 \times 10^{-2}$ & $1.73 \times 10^{-1}$ & $1.63 \times 10^{-2}$ & $1.56 \times 10^{-1}$ & $3.48 \times 10^{-1}$ \\
& HyperDeepONet & $1.54 \times 10^{-4}$ & $1.64 \times 10^{-2}$ & $5.52 \times 10^{-2}$ & $1.83 \times 10^{-2}$ & $1.66 \times 10^{-1}$ & $4.54 \times 10^{-1}$ \\
& HyPINO        & $1.37 \times 10^{-4}$ & $1.54 \times 10^{-2}$ & $5.52 \times 10^{-2}$ & $1.12 \times 10^{-2}$ & $1.30 \times 10^{-1}$ & $3.78 \times 10^{-1}$ \\
& SANO          & \cellcolor[HTML]{C6ECE5}$2.13 \times 10^{-6}$ & \cellcolor[HTML]{C6ECE5}$1.90 \times 10^{-3}$ & \cellcolor[HTML]{C6ECE5}$6.21 \times 10^{-3}$ & \cellcolor[HTML]{C6ECE5}$4.74 \times 10^{-5}$ & \cellcolor[HTML]{C6ECE5}$8.39 \times 10^{-3}$ & \cellcolor[HTML]{C6ECE5}$2.34 \times 10^{-2}$ \\
\midrule

\multirow{7}{*}{\(64\times64\)}
& DeepONet      & $1.46 \times 10^{-1}$ & $4.96 \times 10^{-1}$ & $7.76 \times 10^{-1}$ & $2.13 \times 10^{-1}$ & $5.64 \times 10^{-1}$ & $9.28 \times 10^{-1}$ \\
& FNO           & $4.91 \times 10^{-4}$ & $2.94 \times 10^{-2}$ & $2.14 \times 10^{-1}$ & $1.88 \times 10^{-3}$ & $5.59 \times 10^{-2}$ & $2.80 \times 10^{-1}$ \\
& UNO           & $4.41 \times 10^{-4}$ & $2.79 \times 10^{-2}$ & $2.01 \times 10^{-1}$ & $1.64 \times 10^{-3}$ & $5.19 \times 10^{-2}$ & $2.46 \times 10^{-1}$ \\
& CNO           & $1.32 \times 10^{-3}$ & $4.78 \times 10^{-2}$ & $2.68 \times 10^{-1}$ & $1.91 \times 10^{-2}$ & $1.69 \times 10^{-1}$ & $5.73 \times 10^{-1}$ \\
& HyperDeepONet & $7.97 \times 10^{-4}$ & $3.71 \times 10^{-2}$ & $2.62 \times 10^{-1}$ & $1.89 \times 10^{-2}$ & $1.68 \times 10^{-1}$ & $6.82 \times 10^{-1}$ \\
& HyPINO        & $7.79 \times 10^{-4}$ & $3.71 \times 10^{-2}$ & $2.62 \times 10^{-1}$ & $1.26 \times 10^{-2}$ & $1.39 \times 10^{-1}$ & $6.39 \times 10^{-1}$ \\
& SANO          & \cellcolor[HTML]{C6ECE5}$1.17 \times 10^{-4}$ & \cellcolor[HTML]{C6ECE5}$1.44 \times 10^{-2}$ & \cellcolor[HTML]{C6ECE5}$1.06 \times 10^{-1}$ & \cellcolor[HTML]{C6ECE5}$3.20 \times 10^{-4}$ & \cellcolor[HTML]{C6ECE5}$2.32 \times 10^{-2}$ & \cellcolor[HTML]{C6ECE5}$1.49 \times 10^{-1}$ \\
\midrule

\multirow{7}{*}{\(128\times128\)}
& DeepONet      & $1.59 \times 10^{-1}$ & $5.13 \times 10^{-1}$ & $8.50 \times 10^{-1}$ & $2.55 \times 10^{-1}$ & $6.21 \times 10^{-1}$ & $1.08 \times 10^{0}$ \\
& FNO           & $1.79 \times 10^{-3}$ & $6.21 \times 10^{-2}$ & $4.46 \times 10^{-1}$ & $6.23 \times 10^{-3}$ & $1.15 \times 10^{-1}$ & $6.54 \times 10^{-1}$ \\
& UNO           & $1.46 \times 10^{-3}$ & $5.54 \times 10^{-2}$ & $3.83 \times 10^{-1}$ & $5.40 \times 10^{-3}$ & $1.05 \times 10^{-1}$ & $5.50 \times 10^{-1}$ \\
& CNO           & $3.21 \times 10^{-3}$ & $7.60 \times 10^{-2}$ & $3.89 \times 10^{-1}$ & $3.82 \times 10^{-2}$ & $2.49 \times 10^{-1}$ & $8.23 \times 10^{-1}$ \\
& HyperDeepONet & $2.20 \times 10^{-3}$ & $6.13 \times 10^{-2}$ & $3.56 \times 10^{-1}$ & $4.52 \times 10^{-2}$ & $2.65 \times 10^{-1}$ & $9.53 \times 10^{-1}$ \\
& HyPINO        & $2.00 \times 10^{-3}$ & $5.99 \times 10^{-2}$ & $3.53 \times 10^{-1}$ & $3.46 \times 10^{-2}$ & $2.36 \times 10^{-1}$ & $8.43 \times 10^{-1}$ \\
& SANO          & \cellcolor[HTML]{C6ECE5}$5.59 \times 10^{-4}$ & \cellcolor[HTML]{C6ECE5}$3.16 \times 10^{-2}$ & \cellcolor[HTML]{C6ECE5}$2.03 \times 10^{-1}$ & \cellcolor[HTML]{C6ECE5}$1.35 \times 10^{-3}$ & \cellcolor[HTML]{C6ECE5}$5.40 \times 10^{-2}$ & \cellcolor[HTML]{C6ECE5}$2.89 \times 10^{-1}$ \\
\bottomrule
\end{tabular}
\end{threeparttable}
\end{table*}

Table~\ref{tab:cross_parameter_ac} shows that SANO achieves the best accuracy on all reported metrics under the unseen diffusion coefficient. 
Among the baselines, UNO gives the strongest overall results, but SANO further reduces F-RMSE, RMSE, MSE, Rel-\(L_2\), and Rel-\(H^1\) by approximately \(48.4\%\), \(50.8\%\), \(75.8\%\), \(42.5\%\), and \(48.8\%\), respectively, while using a comparable parameter budget. 
The larger gaps relative to DeepONet, CNO, HyperDeepONet, and HyPINO indicate that merely increasing model flexibility or conditioning capacity is insufficient for stable interpolation across diffusion regimes. 
These results suggest that SANO learns a more transferable operator representation for parameter-dependent PDE dynamics.

\subsection{Cross-resolution Generalization}
\label{subsec:cross_resolution_generalization}
To test whether SANO learns grid-transferable dynamics, we conduct a cross-resolution extrapolation experiment on the 2D AC benchmark, following the broader interest in resolution transfer for spectral neural operators~\cite{DBLP:conf/iclr/LiKALBSA21,pathak2022fourcastnet,zhou2026tf}. 
All models are trained only on \(32\times32\) spatial grids and are then evaluated on \(32\times32\), \(64\times64\), and \(128\times128\) test resolutions. 
This setting examines whether a model trained at a low spatial resolution can generalize to finer grids while maintaining stable long-horizon rollout performance. 
Since the underlying PDE dynamics remain the same across resolutions, performance degradation at higher resolutions mainly reflects the model's ability to transfer its learned operator representation across discretizations.
For cross-resolution evaluation, we report MSE, Rel-\(L_2\), and Rel-\(H^1\) at two rollout checkpoints. 
The notation @1 denotes the first prediction step, which reflects the model's immediate interpolation or extrapolation accuracy after one update. 
The notation @16 denotes the sixteenth autoregressive prediction step, which evaluates error accumulation and rollout stability under repeated model applications. 
Comparing @1 and @16 allows us to assess not only the spatial resolution transfer ability, but also whether the transferred model remains stable over longer prediction horizons.

The cross-resolution results in Table~\ref{tab:cross_resolution_ac} show that SANO maintains the best performance when transferring from the training resolution \(32\times32\) to higher test resolutions. 
At the training resolution \(32\times32\), SANO achieves the lowest errors at both the first and sixteenth rollout steps, with MSE@1 of \(2.13 \times 10^{-6}\) and MSE@16 of \(4.74 \times 10^{-5}\). 
When evaluated at \(64\times64\), SANO still outperforms the strongest baseline in MSE@16, reducing MSE@16 from \(1.64 \times 10^{-3}\) to \(3.20 \times 10^{-4}\). 
This indicates that the learned operator representation can be transferred to a finer grid while retaining stable rollout behavior.
The advantage remains clear under the more challenging \(128\times128\) extrapolation setting. 
Although errors generally grow under harder high-resolution extrapolation, SANO achieves the lowest MSE@16 of \(1.35 \times 10^{-3}\), compared with \(5.40 \times 10^{-3}\) from the strongest baseline in MSE@16. 
It also obtains lower Rel-\(L_2\)@16 and Rel-\(H^1\)@16 (\(5.40 \times 10^{-2}\) and \(2.89 \times 10^{-1}\), respectively), showing that the prediction remains more accurate in both global field amplitude and local gradient structure.

\subsection{Hyperparameter Sensitivity}
\label{app:hyperparameter_sensitivity}

To test whether SANO remains stable under different spatial decompositions and sampling densities, we vary i) the number of subregions \(M\) while fixing the number of sampling points per subregion \(K_i\), and ii) \(K_i\) while fixing \(M\).
We conduct this analysis on the one-dimensional Burgers and two-dimensional Allen--Cahn (AC) benchmarks.
For the AC benchmark, \(M=m_x\times m_y\) denotes an \(m_x\)-by-\(m_y\) subregion partition.
The relative \(L_2\) and relative \(H_1\) errors are reported after the first rollout step and at the final rollout step.

\begin{table*}[t]
    \centering
    \caption{
    Hyperparameter sensitivity on the 1D Burgers and 2D AC benchmarks.
    Increasing the number of subregions \(M\) or the number of sampling points
    per subregion \(K_i\) consistently reduces the relative \(L_2\) and
    relative \(H^1\) errors.
    }
    \label{tab:hyperparameter_sensitivity}
    \begin{tabular}{lllllll}
        \toprule
        Benchmark
        & \(M\)
        & \(K_i\)
        & Rel-\(L_2\) (\(t=1\))
        & Rel-\(H^1\) (\(t=1\))
        & Rel-\(L_2\) (final)
        & Rel-\(H^1\) (final) \\
        \midrule

        \multicolumn{7}{c}{
        \textit{Sensitivity to the number of subregions \(M\)}
        } \\
        \midrule

        1D Burgers
        & \(8\)
        & \(4\)
        & \(8.18\times10^{-3}\)
        & \(4.24\times10^{-2}\)
        & \(1.75\times10^{-2}\)
        & \(8.56\times10^{-2}\) \\

        1D Burgers
        & \(16\)
        & \(4\)
        & \(6.25\times10^{-3}\)
        & \(3.32\times10^{-2}\)
        & \(1.25\times10^{-2}\)
        & \(6.64\times10^{-2}\) \\

        1D Burgers
        & \(32\)
        & \(4\)
        & \(\mathbf{5.52\times10^{-3}}\)
        & \(\mathbf{3.05\times10^{-2}}\)
        & \(\mathbf{9.83\times10^{-3}}\)
        & \(\mathbf{5.95\times10^{-2}}\) \\

        \addlinespace[0.3em]

        2D AC
        & \(4\times4\)
        & \(9\)
        & \(5.28\times10^{-3}\)
        & \(2.10\times10^{-2}\)
        & \(1.04\times10^{-2}\)
        & \(4.28\times10^{-2}\) \\

        2D AC
        & \(8\times8\)
        & \(9\)
        & \(3.47\times10^{-3}\)
        & \(1.43\times10^{-2}\)
        & \(6.94\times10^{-3}\)
        & \(2.85\times10^{-2}\) \\

        2D AC
        & \(12\times12\)
        & \(9\)
        & \(\mathbf{3.08\times10^{-3}}\)
        & \(\mathbf{1.27\times10^{-2}}\)
        & \(\mathbf{5.53\times10^{-3}}\)
        & \(\mathbf{2.59\times10^{-2}}\) \\

        \midrule
        \multicolumn{7}{c}{
        \textit{Sensitivity to the number of sampling points \(K_i\)}
        } \\
        \midrule

        1D Burgers
        & \(16\)
        & \(2\)
        & \(7.45\times10^{-3}\)
        & \(3.98\times10^{-2}\)
        & \(1.60\times10^{-2}\)
        & \(7.94\times10^{-2}\) \\

        1D Burgers
        & \(16\)
        & \(4\)
        & \(6.25\times10^{-3}\)
        & \(3.32\times10^{-2}\)
        & \(1.25\times10^{-2}\)
        & \(6.64\times10^{-2}\) \\

        1D Burgers
        & \(16\)
        & \(8\)
        & \(\mathbf{5.09\times10^{-3}}\)
        & \(\mathbf{2.66\times10^{-2}}\)
        & \(\mathbf{8.69\times10^{-3}}\)
        & \(\mathbf{5.31\times10^{-2}}\) \\

        \addlinespace[0.3em]

        2D AC
        & \(8\times8\)
        & \(4\)
        & \(4.11\times10^{-3}\)
        & \(1.72\times10^{-2}\)
        & \(8.24\times10^{-3}\)
        & \(3.44\times10^{-2}\) \\

        2D AC
        & \(8\times8\)
        & \(9\)
        & \(3.47\times10^{-3}\)
        & \(1.43\times10^{-2}\)
        & \(6.94\times10^{-3}\)
        & \(2.85\times10^{-2}\) \\

        2D AC
        & \(8\times8\)
        & \(16\)
        & \(\mathbf{3.14\times10^{-3}}\)
        & \(\mathbf{1.29\times10^{-2}}\)
        & \(\mathbf{6.32\times10^{-3}}\)
        & \(\mathbf{2.54\times10^{-2}}\) \\

        \bottomrule
    \end{tabular}
\end{table*}

\paragraph{Sensitivity to the Number of Subregions.}
Table~\ref{tab:hyperparameter_sensitivity} shows that increasing \(M\) consistently improves both one-step and final-step accuracy.
On the Burgers benchmark, increasing \(M\) from \(8\) to \(32\) reduces the final relative \(L_2\) error from \(1.75\times10^{-2}\) to \(9.83\times10^{-3}\), corresponding to a \(43.8\%\) reduction.
The final relative \(H^1\) error decreases from \(8.56\times10^{-2}\) to \(5.95\times10^{-2}\), a \(30.5\%\) reduction.
On the AC benchmark, refining the partition from \(4\times4\) to \(12\times12\) reduces the final relative \(L_2\) and relative \(H^1\) errors from \(1.04\times10^{-2}\) to \(5.53\times10^{-3}\) and from \(4.28\times10^{-2}\) to \(2.59\times10^{-2}\), respectively.
These reductions correspond to \(46.8\%\) and \(39.5\%\).
The consistent gains indicate that a finer decomposition allows the spatially varying operator parameters to resolve local solution behavior more accurately.

\paragraph{Sensitivity to the Number of Sampling Points.}
Increasing \(K_i\) also reduces all four error measures on both benchmarks.
For Burgers, increasing \(K_i\) from \(2\) to \(8\) reduces the final relative \(L_2\) error from \(1.60\times10^{-2}\) to \(8.69\times10^{-3}\), a \(45.7\%\) reduction, while the final relative \(H_1\) error decreases from \(7.94\times10^{-2}\) to \(5.31\times10^{-2}\), a \(33.1\%\) reduction.
For AC, increasing \(K_i\) from \(4\) to \(16\) reduces the final relative \(L_2\) and relative \(H_1\) errors from \(8.24\times10^{-3}\) to \(6.32\times10^{-3}\) and from \(3.44\times10^{-2}\) to \(2.54\times10^{-2}\), respectively.
These results show that denser sampling improves parameter interpolation within each subregion, although the smaller gains on AC indicate diminishing returns once the local sampling density is sufficiently high.

All tested configurations remain within the same error scale, while finer subregion partitions and denser local sampling produce monotonic improvements.
The default configurations, \(M=16\) and \(K_i=4\) for Burgers and \(M=8\times8\) and \(K_i=9\) for AC, provide a balanced trade-off between approximation accuracy and computational cost.
The sensitivity results therefore support the stability of SANO with respect to both spatial decomposition and local sampling density.

\subsection{Complex-Geometry Evaluation on HZ-G}
\label{app:hz_g_details}

\paragraph{Problem Setup}
To test whether spatially adaptive local operators, we consider the HZ-G benchmark, a two-dimensional modified Helmholtz problem defined on a perforated domain.
The scalar field \(u:\Omega\rightarrow\mathbb{R}\) satisfies
\begin{equation}
-\Delta u(x,y)+k^2u(x,y)=f(x,y),
\qquad (x,y)\in\Omega,
\label{eq:hz_g}
\end{equation}
where
\begin{equation}
\Omega
=
[-1,1]^2
\setminus
\bigcup_{i=1}^{4}R_i
\end{equation}
contains four circular holes.
The source term is
\begin{equation}
f(x,y)
=
A\mu_2^2y
\sin(\mu_1\pi x)
\sin(\mu_2\pi y),
\end{equation}
with
\begin{equation}
\mu_1=1,
\qquad
\mu_2=4,
\qquad
k=8,
\qquad
A=10.
\end{equation}
The outer boundary satisfies \(u=0.2\), whereas all hole boundaries satisfy \(u=1.0\).
This configuration introduces boundary-induced spatial heterogeneity and therefore tests whether a model can adapt its local mapping near irregular internal boundaries.

\paragraph{Baseline Adaptation.}
Standard neural operators such as FNO, UNO, CNO, DeepONet, HyperDeepONet, and HyPINO are not directly defined on the perforated domain.
For a consistent comparison, we embed \(\Omega\) into the bounding square \([-1,1]^2\) and encode the problem on a regular grid as
\begin{equation}
X(x,y)
=
\left[
f(x,y),
m(x,y),
b(x,y),
g(x,y),
x,
y
\right],
\label{eq:hz_g_input}
\end{equation}
where \(m\) is the domain mask, \(b\) is the boundary indicator, and \(g\) stores the prescribed Dirichlet values.
All models receive the same source, geometry, boundary, and coordinate information.

For grid-based models, the supervised loss is evaluated only over the valid physical domain:
\begin{equation}
\mathcal{L}_{\mathrm{data}}
=
\frac{
\sum_x m(x)
\left|
\widehat{u}(x)-u(x)
\right|^2
}{
\sum_x m(x)
},
\end{equation}
and the boundary loss is
\begin{equation}
\mathcal{L}_{\mathrm{bc}}
=
\frac{
\sum_x b(x)
\left|
\widehat{u}(x)-g(x)
\right|^2
}{
\sum_x b(x)
}.
\end{equation}
HyPINO additionally minimizes the residual
\begin{equation}
r(x)
=
-\Delta\widehat{u}(x)
+
k^2\widehat{u}(x)
-
f(x)
\end{equation}
over interior points whose finite-difference stencil remains inside \(\Omega\).
SANO uses the same input information but generates a spatially varying code, allowing the local mapping to change near hole boundaries and smoother interior regions.

\subsection{Perforated-Domain Laplace Benchmark}
\label{app:ps_c_details}

\paragraph{Problem Setup.}
To test whether spatially adaptive local operators in a source-free elliptic problem, we consider the two-dimensional Poisson problem with four circular holes (PS-C).
Because the source term vanishes, PS-C reduces to the Laplace equation
\begin{equation}
-\Delta u(x,y)=0,
\qquad (x,y)\in\Omega,
\label{eq:ps_c}
\end{equation}
where
\begin{equation}
\Delta u
=
\frac{\partial^2 u}{\partial x^2}
+
\frac{\partial^2 u}{\partial y^2}.
\end{equation}
The perforated domain is
\begin{equation}
\Omega
=
[-0.5,0.5]^2
\setminus
\bigcup_{i=1}^{4}R_i,
\label{eq:ps_c_domain}
\end{equation}
with four circular holes
\begin{align}
R_1
&=
\left\{
(x,y):
(x-0.3)^2+(y-0.3)^2\leq 0.1^2
\right\},\\
R_2
&=
\left\{
(x,y):
(x+0.3)^2+(y-0.3)^2\leq 0.1^2
\right\},\\
R_3
&=
\left\{
(x,y):
(x-0.3)^2+(y+0.3)^2\leq 0.1^2
\right\},\\
R_4
&=
\left\{
(x,y):
(x+0.3)^2+(y+0.3)^2\leq 0.1^2
\right\}.
\end{align}
The outer and inner boundaries satisfy
\begin{equation}
u=1
\quad
\text{on }
\partial[-0.5,0.5]^2,
\qquad
u=0
\quad
\text{on }
\partial R_i,
\quad i=1,\ldots,4.
\label{eq:ps_c_bc}
\end{equation}
Under the general second-order linear form
\begin{equation}
c_0u+c_1u_x+c_2u_y+c_3u_{xx}+c_4u_{yy}=f,
\end{equation}
PS-C corresponds to
\begin{equation}
\mathbf{c}
=
(0,0,0,-1,-1),
\qquad
f(x,y)=0.
\end{equation}

The solution is determined entirely by the potential difference between the outer boundary and the four internal boundaries.
PS-C therefore isolates the model's ability to propagate boundary information through a multiply connected domain and resolve the localized gradients surrounding the holes.
All baselines use the same mask-based geometric encoding and loss construction described in Sec.~\ref{subsec:hz_g}.

\subsection{Qualitative Rollout Visualizations}
\label{app:visualization}

To assess whether the predicted fields preserve physically meaningful structures, we visualize rollouts produced by the SANO model on five partial differential equation (PDE) benchmarks:
the one-dimensional Burgers equation, the one-dimensional Kuramoto--Sivashinsky (KS) equation,
the two-dimensional Allen--Cahn (AC) equation,
the two-dimensional Fisher--Kolmogorov--Petrovsky--Piskunov (FKPP) equation,
and three-dimensional compressible flow (3D CF).
The Burgers and KS visualizations present predicted space--time solution fields.
The AC and FKPP visualizations report temporal snapshots at
\(t=0,5,15,\) and \(20\), whereas the 3D CF visualization reports pressure and density fields at
\(t=0,5,10,15,\) and \(20\).
Each comparison includes the ground-truth field, model predictions, and the corresponding absolute-error maps, enabling direct assessment of structural preservation, interface and front evolution, and error propagation across different PDE dynamics.

\subsubsection{One-Dimensional Burgers and Kuramoto--Sivashinsky Equations}

Figure~\ref{fig:vis_burgers} compares the predicted space--time fields and corresponding absolute-error maps on the one-dimensional Burgers benchmark.
SANO reproduces the dominant structures of the ground-truth solution more faithfully than the competing methods.
Its errors remain weaker and less spatially concentrated, whereas several baselines exhibit larger deviations from the reference field.
Figure~\ref{fig:vis_ks} reports the corresponding comparison on the one-dimensional KS benchmark.
The KS equation exhibits phase-sensitive spatiotemporal dynamics with substantial local variation, making long-horizon prediction sensitive to accumulated phase errors.
SANO preserves the dominant spatiotemporal patterns more faithfully, whereas several baselines produce distorted structures and visible phase deviations.
The corresponding error maps remain weaker and less spatially concentrated for SANO over the space--time domain.

\subsubsection{Two-Dimensional Allen--Cahn and Fisher--KPP Equations}

Figure~\ref{fig:vis_ac_fisher} reports rollout predictions on the two-dimensional AC and FKPP benchmarks at
\(t=0,5,15,\) and \(20\).
On the AC benchmark, SANO preserves the evolving interfaces more accurately and produces weaker absolute-error patterns than the competing methods.
Several baselines instead develop distorted interfaces and larger localized errors at later rollout steps.

On the FKPP benchmark, SANO more accurately tracks the propagating fronts and maintains closer agreement with the reference solution throughout the rollout.
Its error maps remain relatively weak across the reported time steps, whereas several competing methods develop larger and more structured errors.
These observations indicate that SANO preserves reaction--diffusion dynamics more faithfully over long rollouts.

\subsubsection{Three-Dimensional Compressible Flow}

Figure~\ref{fig:vis_ns3d} reports rollouts on the 3D CF benchmark at
\(t=0,5,10,15,\) and \(20\) for the pressure and density components.
At each time step, the \(XY\), \(XZ\), and \(YZ\) slices are reported together with the corresponding predictions and absolute-error maps.
SANO better preserves the dominant pressure and density structures across spatial slices and rollout steps.
Its errors remain weaker over most slices, whereas several baselines develop more pronounced structured errors at later steps.
These results indicate that SANO improves component-wise prediction accuracy and rollout stability for high-dimensional fluid states.

\begin{figure*}[t]
    \centering
    \includegraphics[width=\linewidth]{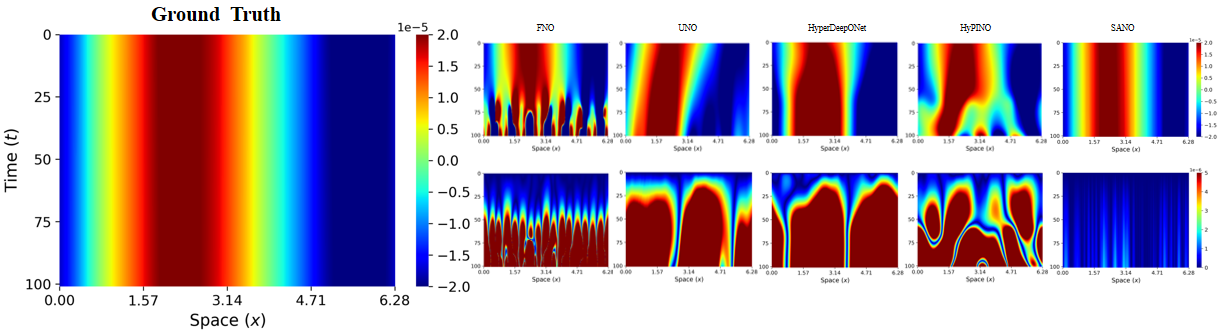}
    \caption{
    1D Burgers dynamics: ground-truth space--time solution, model predictions, and absolute-error maps.
    SANO preserves the dominant solution structures while producing weaker and less spatially concentrated errors than the competing methods.
    }
    \label{fig:vis_burgers}
\end{figure*}

\begin{figure*}[t]
    \centering
    \includegraphics[width=\linewidth]{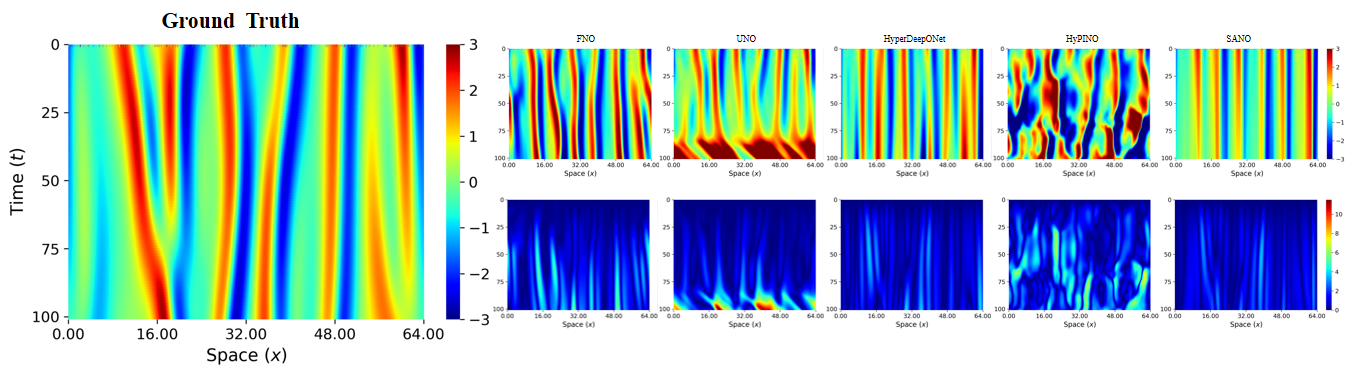}
    \caption{
    1D KS dynamics: ground-truth space--time solution, model predictions, and absolute-error maps.
    SANO reduces structural distortion and accumulated phase error over the rollout.
    }
    \label{fig:vis_ks}
\end{figure*}

\begin{figure*}[h]
    \centering
    \subfloat[2D AC benchmark]{
        \includegraphics[width=0.98\linewidth]{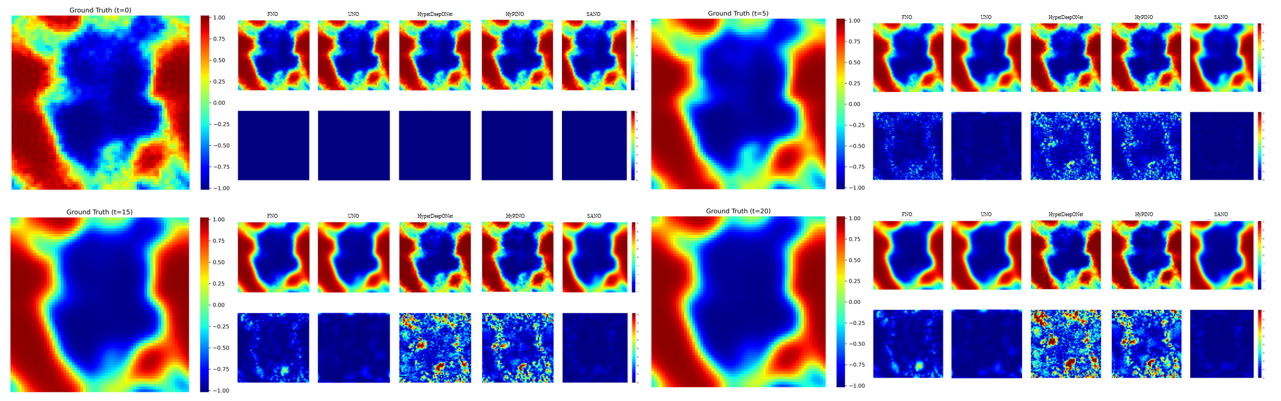}
        \label{fig:ac_visualization}
    }

    \vspace{0.3em}

    \subfloat[2D FKPP benchmark]{
        \includegraphics[width=0.98\linewidth]{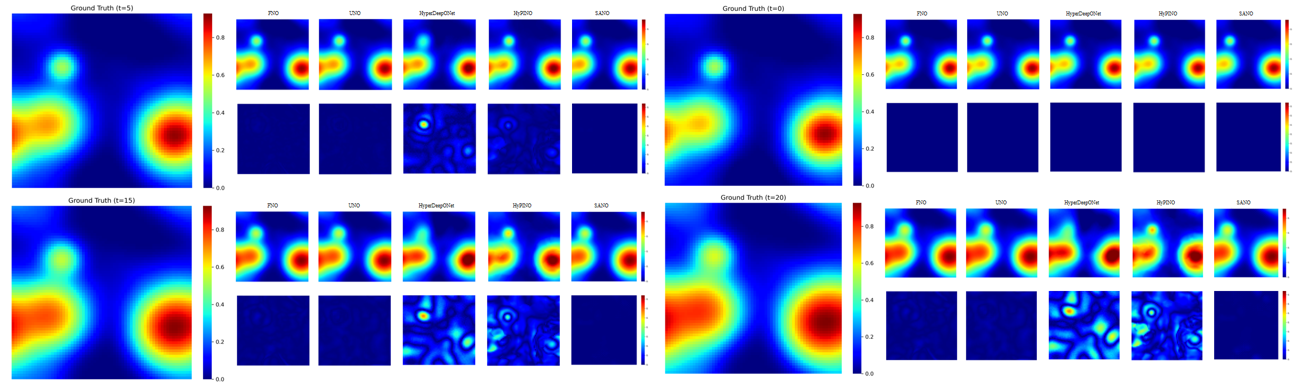}
        \label{fig:fisher_visualization}
    }
    \caption{
    Reaction--diffusion rollouts at \(t=0,5,15,\) and \(20\).
    (top) AC phase-field predictions and absolute-error maps.
    (bottom) FKPP concentration-field predictions and absolute-error maps.
    SANO better preserves evolving interfaces and propagating fronts throughout the rollout.
    }
    \label{fig:vis_ac_fisher}
\end{figure*}

\begin{figure*}[t]
    \centering
    \subfloat[Pressure]{
        \includegraphics[width=0.40\linewidth]{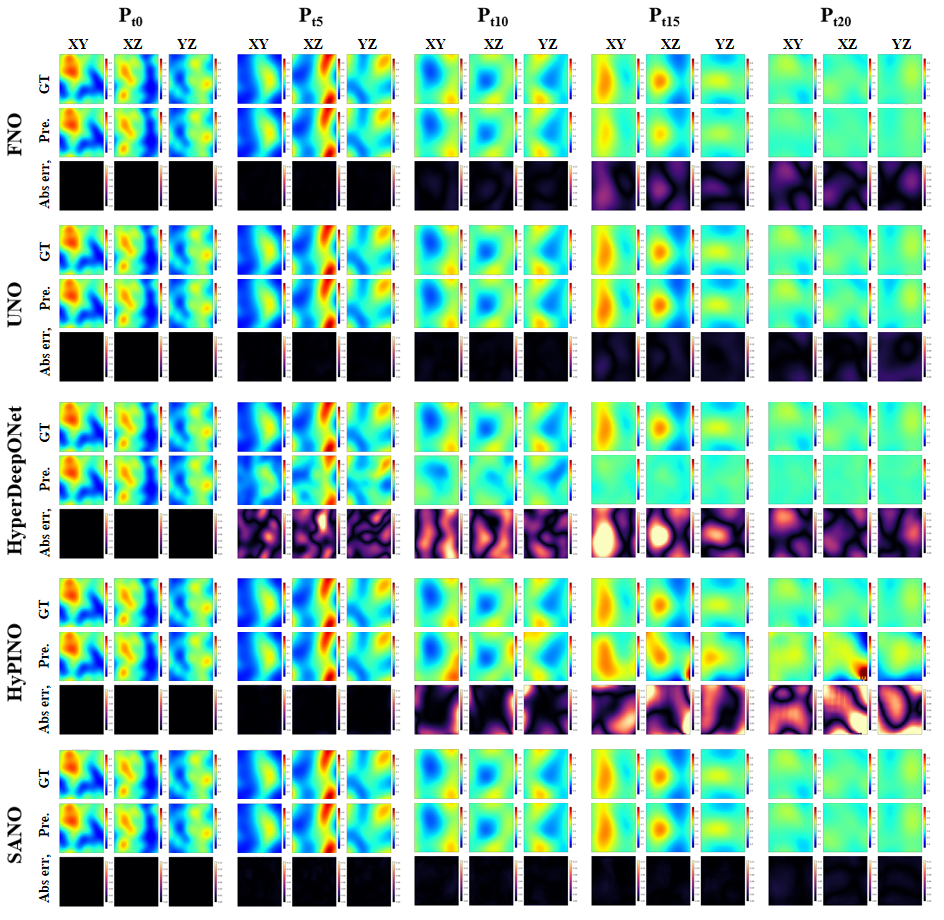}
        \label{fig:ns_pressure}
    }

    \vspace{0.3em}

    \subfloat[Density]{
        \includegraphics[width=0.40\linewidth]{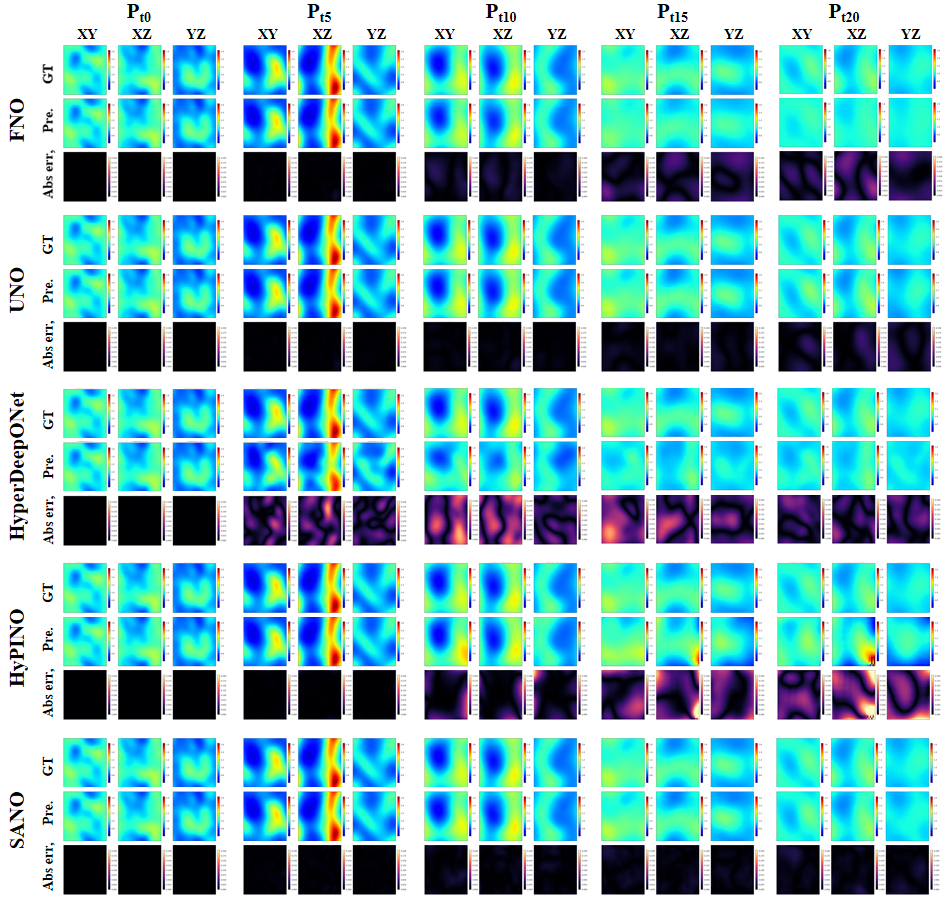}
        \label{fig:ns_density}
    }
    \caption{
    3D CF rollouts at \(t=0,5,10,15,\) and \(20\).
    (top) Pressure-field predictions and absolute-error maps on the \(XY\), \(XZ\), and \(YZ\) slices.
    (bottom) Density-field predictions and absolute-error maps on the same slices.
    SANO preserves the dominant component-wise structures while limiting structured error growth at later rollout steps.
    }
    \label{fig:vis_ns3d}
\end{figure*}

\end{document}